\documentclass{article}

\usepackage[preprint]{icml2026}

\usepackage[utf8]{inputenc}
\usepackage[T1]{fontenc}
\usepackage{kotex}

\usepackage{microtype}
\usepackage{graphicx}
\usepackage{subcaption}
\usepackage{booktabs}
\usepackage{url}
\usepackage{amsfonts}
\usepackage{nicefrac}
\usepackage{makecell}
\usepackage[table]{xcolor}
\usepackage{enumitem}
\usepackage{caption}
\usepackage{wrapfig}
\usepackage{array}
\usepackage[most]{tcolorbox}
\usepackage{listings}
\usepackage{xspace}
\usepackage{pifont}

\usepackage{amsmath}
\usepackage{amssymb}
\usepackage{mathtools}
\usepackage{amsthm}
\usepackage[framemethod=TikZ]{mdframed}

\usepackage{hyperref}
\usepackage[capitalize,noabbrev]{cleveref}

\usepackage[textsize=tiny]{todonotes}
\usepackage{placeins}
\usepackage{fancyhdr}
\usepackage{placeins}

\makeatletter
\let\ICMLorigPrintAffiliationsAndNotice\printAffiliationsAndNotice
\renewcommand{\printAffiliationsAndNotice}[1]{%
  \begingroup
  \long\def\@footnotetext##1{}%
  \long\def\@mpfootnotetext##1{}%
  \ICMLorigPrintAffiliationsAndNotice{#1}%
  \endgroup
}
\makeatother

\definecolor{algbg}{HTML}{ECECFF}
\definecolor{algframe}{HTML}{9999CC}
\definecolor{lightgray}{cmyk}{0,0,0,.05}

\newtcolorbox{promptbox}[1]{
  enhanced,
  breakable,
  colback=gray!10,
  colframe=gray!55,
  coltitle=white,
  colbacktitle=gray!55,
  title=\textbf{#1},
  fonttitle=\bfseries,
  fontupper=\footnotesize,
  arc=2mm,
  boxrule=0.8pt,
  left=2mm,
  right=2mm,
  top=2mm,
  bottom=2mm,
  before skip=6pt,
  after skip=6pt,
  pad at break*=2mm
}

\newtcolorbox{contributionbox}{
  enhanced,
  colback=cyan!3,
  colframe=cyan!3,
  borderline west={1.2pt}{0pt}{cyan!60!black},
  sharp corners,
  boxrule=0pt,
  left=5pt,
  right=4pt,
  top=4pt,
  bottom=4pt
}

\newtcolorbox{namedbox}[2][]{enhanced,breakable,
  colback=cyan!4,
  colframe=cyan!65!black,
  boxed title style={colback=cyan!65!black, colframe=cyan!65!black},
  coltitle=white,
  fonttitle=\small,
  title={#2},
  attach boxed title to top left={xshift=1.0mm, yshift*=-\tcboxedtitleheight/2},
  boxsep=0.1pt,
  left=2pt,
  right=2pt,
  top=4pt,
  bottom=4pt,
  arc=3pt,
  #1
}

\newcolumntype{C}[1]{>{\centering\arraybackslash}p{#1}}

\newcommand{\proposed}{\texttt{ShaQ}\xspace}

\newlength\GreyboxOuterVspace
\newlength\GreyboxPadding
\mdfdefinestyle{mathblockframe}{
  linecolor=gray!40,
  linewidth=0.5pt,
  backgroundcolor=gray!8,
  roundcorner=3pt,
  innertopmargin=3pt,
  innerbottommargin=3pt,
  innerleftmargin=3pt,
  innerrightmargin=3pt,
  skipabove=3pt,
  skipbelow=3pt
}

\newmdtheoremenv[style=mathblockframe]{theorem}{Theorem}[section]
\newmdtheoremenv[style=mathblockframe]{proposition}{Proposition}[section]
\newmdtheoremenv[style=mathblockframe]{lemma}{Lemma}[section]
\newmdtheoremenv[style=mathblockframe]{corollary}{Corollary}[section]
\newmdtheoremenv[style=mathblockframe]{property}{Property}[section]

\theoremstyle{definition}
\newmdtheoremenv[style=mathblockframe]{definition}{Definition}[section]
\newmdtheoremenv[style=mathblockframe]{assumption}{Assumption}[section]
\newmdtheoremenv[style=mathblockframe]{remark}{Remark}[section]

\begin{document}

\onecolumn


\pagestyle{plain}
\thispagestyle{plain}

\begin{center}
{\LARGE \bfseries
Localizing Input Uncertainty Quantification for\\
Large Language Models via Shapley Values
\par}

\vspace{0.5cm}

{\large
Seongjun Lee\textsuperscript{1,*}
\quad
Suwan Yoon\textsuperscript{1,*}
\quad
Changhee Lee\textsuperscript{1,\textdagger}
\par}

{\small
\textsuperscript{1}Department of Artificial Intelligence, Korea University\\
\texttt{\{pyoung7307,suwanyoon,changheelee\}@korea.ac.kr}
\par}
\end{center}

\begingroup
\renewcommand{\thefootnote}{\fnsymbol{footnote}}
\footnotetext[1]{Equal contribution. \quad \textsuperscript{\textdagger}Corresponding author.}
\endgroup

\vspace{0.3in}

\begin{abstract}
    As large language models (LLMs) are increasingly integrated into high-stakes decision-making, the ability to reliably quantify uncertainty has become a critical requirement for safety and trust. However, current uncertainty quantification methods primarily operate at the output level, often failing to distinguish whether uncertainty arises from the model’s lack of knowledge or from ambiguity in the user’s input. While input-centric uncertainty quantification has recently emerged as a promising direction, it remains relatively underexplored and typically relies on coarse, input-level information. Consequently, users are provided with scalar uncertainty scores that offer little actionable guidance on \textit{which parts of the input} should be clarified to improve reliability. To address this limitation, we propose \textbf{Sha}pley-based input uncertainty \textbf{Q}uantification (\proposed), a framework for \textit{span-level attribution} of input-induced uncertainty. Our approach models ambiguous spans in the input as players in a cooperative game and quantifies their contributions using Shapley values, defined via the weighted average of marginal reductions in conditional entropy obtained by clarifying each span coalition. Unlike existing input-level approaches, our formulation explicitly captures complex interactions among spans and provides a principled decomposition in which individual attributions sum exactly to the total input-induced uncertainty. We evaluate \proposed on the AmbigQA and AmbiEnt benchmarks, where it achieves state-of-the-art performance in ambiguity detection. We further demonstrate its practical utility on a MediTOD benchmark, showing that \proposed can precisely localize under-specified clinical utterances, facilitating more effective human-AI collaboration in high-stakes settings. Overall, our results show that \proposed not only improves uncertainty estimation but also provides actionable insights for targeted input clarification. Codes are available \textcolor{blue}{\underline{\href{https://anonymous.4open.science/r/ShaQ-0E39/README.md}{here}}}.
\end{abstract}
\printAffiliationsAndNotice{}
\FloatBarrier
\section{Introduction}\label{sec:intro}
\begin{figure}[t]
\centering
\includegraphics[width=0.80\linewidth]{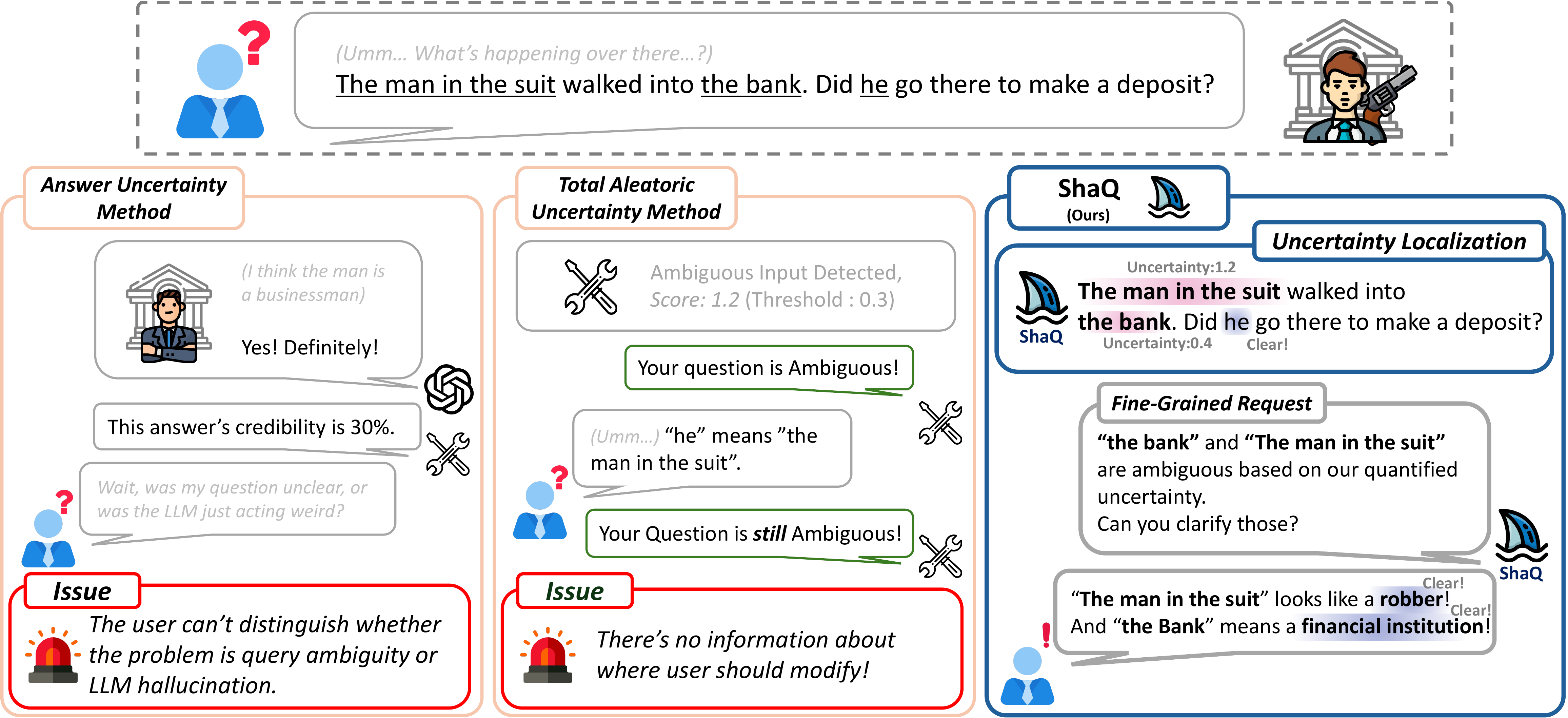} \vspace{-1mm}
\caption{\textbf{Comparison of uncertainty methods to LLM-user interaction.} (\textbf{Left}) Output-level uncertainty methods provide confidence scores, but do not distinguish whether uncertainty arises from input ambiguity or model hallucination. (\textbf{Center}) Aggregate aleatoric uncertainty methods identify ambiguous inputs, but lack localization. 
(\textbf{Right}) \proposed\ localizes uncertainty to individual spans, enabling targeted clarification of the specific sources of ambiguity.}
\label{fig:motivation}
\vspace{-5mm}
\end{figure}

Large language models (LLMs) have demonstrated strong zero-shot performance across a wide range of tasks and are increasingly deployed as decision-support tools in many real-world settings~\citep{llmzeroshot_kojima_nips, llmfewshot_brown_nips, humancentered_devic_arxiv, HumanDecisionMaking_marusich_arxiv}. However, they often produce confident responses even when the input is ambiguous or under-specified~\citep{underspecifiedChat_herlihy_arxiv, clam_kuhn_icmlworkshop}, making it critical to quantify the uncertainty in their outputs for trustworthy deployment.
This uncertainty can be formalized within the classical uncertainty quantification (UQ) framework~\citep{weightuncertainty_blundell_icml,dropout_gal_icml,heteroscadistiy_kendall_nips,priornet_malinin_nips,deepensemble_lakshiminarayana_nips,deepensembleloss_fort_arxiv}, which decomposes predictive uncertainty into \textit{epistemic uncertainty} -- stemming from incomplete knowledge of model parameters -- and \textit{aleatoric uncertainty} -- arising from ambiguity or noise intrinsic to the input~\citep{uncertaintyDL_gal_cambridge}. Building on this framework, a growing body of work has focused on estimating uncertainty in LLM outputs using signals such as token-level likelihoods, semantic consistency across sampled responses, or model parameter perturbations~\citep{autoregressiveUE_malinin_arxiv,tokur_zhang_arxiv,spuq_gao_acl,SemanticUncertainty_kuhn_iclr,cocoa_vashurin_nips,uqconf_lin_tmlr}. 
While these approaches have proven effective for detecting hallucinations and confabulations~\citep{selfcheckgpt_manakul_emnlp,detectHallucination_farquhar_nature}, they remain fundamentally limited: a scalar uncertainty score indicates that an output is uncertain, but provides no actionable insight into \textit{which part of the input} should be revised to resolve that ambiguity.

This limitation highlights a fundamental gap that reducing uncertainty ultimately requires understanding its origin within the input itself. However, input-centric uncertainty analysis remains largely underexplored~\citep{uqsurvey_liu_kdd}. The need for such analysis is especially pronounced in practical deployments, where LLMs are accessed via black-box APIs and modifying model parameters to reduce epistemic uncertainty is infeasible. In these settings, improving reliability must instead rely on resolving the \textit{aleatoric uncertainty} inherent in the input. Yet, existing approaches to input-level uncertainty provide only an aggregate score for the entire input~\citep{ICE_hou_arxiv}, offering little guidance on which parts of the input contribute to the uncertainty. As illustrated in Figure~\ref{fig:motivation}, the query \textit{``The man in the suit walked into the bank. Did he go there to make a deposit?''} contains multiple ambiguous expressions: \textit{the man} may refer to either a \textit{robber} or a \textit{customer}, while \textit{the bank} may denote a \textit{financial institution} or a \textit{riverbank}. These ambiguous spans can interact in complex ways, further compounding uncertainty in the model’s predictions. Without localization, users are often forced to clarify the entire input without knowing where ambiguity originates, leading to unnecessary effort and limited practical benefit. 

To address this gap, we propose \textit{\textbf{Sha}pley-based input uncertainty \textbf{Q}uantification} (\proposed), a framework that localizes input-induced uncertainty to individual ambiguous spans within an input. \proposed operates by resolving each span's ambiguity through \textit{premises} -- declarative statements that disambiguate a span's semantic role -- and measuring the consequent reduction in predictive entropy. The fundamental challenge, as illustrated in our running example, is that ambiguous spans are often interdependent: the uncertainty resolved by one span is inherently conditioned on which other spans have already been clarified. This complex interaction structure motivates a Shapley formulation, where ambiguous spans act as \textit{players} and an \textit{information-theoretic} value function defines their coalitional utility in a cooperative game. To compute these values, we introduce a \textit{bottom-up marginalization} algorithm that evaluates marginal contributions across all possible coalitions while provably guaranteeing non-negative uncertainty estimates. The resulting framework provides a principled and fine-grained decomposition of input-induced uncertainty, identifying \textit{where} ambiguity arises and enabling users to focus their clarification efforts precisely where they are most needed.

\begin{contributionbox}
\textbf{Our contributions are summarized as follows:} \vspace{-1mm}
\begin{itemize}[leftmargin=1.5em]
    \item We propose \proposed, a novel Shapley-based framework for localizing the input-induced uncertainty of LLM inputs at the span level.
    \item We theoretically show that \proposed generalizes existing input-level uncertainty quantification methods~\citep{ICE_hou_arxiv}, providing a principled decomposition into span-level contributions.
    \item We conduct experiments on AmbigQA~\citep{ambigqa_sewon_emnlp} and AmbiEnt~\citep{ambient-liu-etal-2023-afraid}, where \proposed achieves state-of-the-art ambiguity detection, and demonstrate its practical utility on MediTOD~\citep{meditod_saley_emnlp} in real-world medical dialogues.
\end{itemize}
\end{contributionbox}

\section{Related Works}\label{sec:rel}
\textbf{Uncertainty Quantification.}~~
Early work on uncertainty quantification (UQ) focused primarily on classification settings, estimating predictive uncertainty via posterior approximations, dropout, ensembles, or architectural modifications~\citep{dropout_gal_icml,deepensemble_lakshiminarayana_nips,deepensembleloss_fort_arxiv,heteroscadistiy_kendall_nips,priornet_malinin_nips}. However, these approaches do not transfer readily to modern LLMs, which involve autoregressive generation, massive parameter scales, and black-box access. Consequently, recent LLM UQ research has pursued several alternative directions: exploiting token likelihoods or output entropy~\citep{autoregressiveUE_malinin_arxiv}, measuring semantic consistency across multiple sampled outputs~\citep{SemanticUncertainty_kuhn_iclr,detectHallucination_farquhar_nature,cocoa_vashurin_nips,uqconf_lin_tmlr}, assessing prediction stability under input or model perturbations~\citep{spuq_gao_acl,tokur_zhang_arxiv}, or eliciting verbalized confidence scores directly from the model~\citep{ask4conf_tian_emnlp}. Our framework departs from this perspective: rather than quantifying uncertainty in a generated answer, we localize and quantify the sources of ambiguity that cause an input to admit multiple valid answers.

\noindent
\textbf{Shapley Values for Explanation and Uncertainty.}~~
Shapley values originate from cooperative game theory as a principled solution for fairly distributing a coalition's payoff among players~\citep{shapleyvalue_shapley_american}. SHAP introduced this framework to machine learning by attributing a model's prediction to individual input features~\citep{SHAP_lundberg_nips,shapAlgorithm_suinlee_nature}. Recent work extends this idea to uncertainty attribution by replacing the payoff with information-theoretic quantities such as predictive or conditional entropy~\citep{infoshap_watson_nips}, but remains confined to supervised settings with fixed feature spaces and tractable predictive distributions. How to define and estimate Shapley values for uncertainty quantification in LLMs -- accounting for black-box access, semantic equivalence, and free-form generation -- remains underexplored.
\vspace{-0.5em}
\section{Problem Formulation}\label{sec:problem}

\textbf{Notation.}~~
Let $\mathcal{D}$ denote an observed dataset $\mathcal{D}$ composed of inputs $X \in \mathcal{X}$ and corresponding outputs $Y \in \mathcal{Y}$, respectively. 
Following standard conventions, we use uppercase letters for random variables and lowercase letters for their realizations. For notational convenience, we omit the explicit conditioning on $\mathcal{D}$ and $\theta$ when it is clear from context.

\subsection{Input Uncertainty Quantification in LLMs}
A model's predictive uncertainty is conventionally decomposed into \emph{epistemic} uncertainty, stemming from incomplete knowledge of model parameters, and \emph{aleatoric} uncertainty, arising from ambiguity intrinsic to the input~\citep{uncertaintyDL_gal_cambridge}. Reducing epistemic uncertainty requires access to model parameters, which is infeasible in black-box LLM deployments.
To circumvent these limitations, recent work shifts the source of uncertainty from the parameter space to the input space~\citep{ICE_hou_arxiv}. Specifically, the model parameters are fixed to $\theta$, and uncertainty is instead characterized through a latent variable $C$ that represents plausible clarifications of an ambiguous or under-specified input $x$.
The predictive uncertainty for a given input $x$ is then quantified by the entropy $\mathcal{H}(Y \mid x)$, that can be decomposed as
\begin{equation}\label{eq:input_decomp}
    \mathcal{H}(Y \mid x) = \underbrace{\mathcal{I}(Y;C \mid x)}_{\text{input-induced uncertainty}} + \underbrace{\mathbb{E}_{c \sim q(C \mid x)} 
    \bigl[\mathcal{H}(Y \mid x, c)\bigr]}_{\text{remaining uncertainty}},
\end{equation}
where $q(C \mid x)$ denotes an approximate distribution over plausible clarifications given $x$. The mutual information term in \eqref{eq:input_decomp} captures the uncertainty induced by input ambiguity, and following previous work \citep{ICE_hou_arxiv}, we treat it as an estimate of the \emph{aleatoric uncertainty} of the input.
The expected conditional entropy term quantifies the residual uncertainty after clarification, and under the assumption that input ambiguity is the primary source of aleatoric uncertainty, it serves as a proxy for the model's irreducible uncertainty. 
However, \eqref{eq:input_decomp} treats aleatoric uncertainty as a monolithic property of the entire input, offering no means to attribute individual contributions of each input component.


\subsection{Localizing Aleatoric Uncertainty}
In this work, we aim to localize aleatoric uncertainty at a finer granularity by attributing it to specific components of the input. A central challenge is that, unlike structured data with predefined feature boundaries, natural language lacks a canonical decomposition into non-overlapping semantic units. Consequently, quantifying how different parts of the input contribute to overall uncertainty first requires identifying the units over which ambiguity arises.

Importantly, ambiguity is often context-dependent: a model’s prediction can vary depending on how multiple phrases are jointly interpreted. This indicates that uncertainty is not solely a property of individual terms, but also emerges from interactions among them. To capture these dynamics, we focus on contiguous phrases that form the minimal, semantically coherent units capable of supporting multiple plausible interpretations. We refer to these functional units as \emph{ambiguous spans}:
\begin{definition}[\textbf{Ambiguous Span}]\label{def:span}
An \emph{ambiguous span} is a contiguous subsequence of the input sequence that constitutes a minimal semantically coherent unit admitting multiple plausible interpretations. We denote by $\mathcal{S} = \{s_1, \ldots, s_n\}$ a collection of non-overlapping ambiguous spans extracted from $x$.
\end{definition} \vspace{-1.5mm}
Given an ambiguous span collection $\mathcal{S} = \{s_1, \ldots, s_n\}$, we associate each span $s_k$ with a \textit{clarification} random variable $C_k$, whose distribution $q(C_k \mid x)$ characterizes the set of plausible interpretations of $s_k$. We denote the joint clarification as $C = (C_1, \dots, C_n)$ and the joint clarification excluding the $k$-th span as $C_{\setminus k}$. We use $c$, $c_k$, and $c_{\setminus k}$ to denote corresponding realizations, respectively.

A natural approach to quantifying the contribution of a span $s_k$ is to measure how much information about $Y$ is provided by its clarification $C_k$ once all other ambiguities have been resolved. We formalize this intuition using a leave-one-out (LOO) attribution:
\begin{definition}[\textbf{LOO Span Attribution}]\label{def:loo}
The LOO attribution of span $s_k$ is defined as the conditional mutual information
\begin{equation} \label{eq:span_loo} 
\phi^{\mathrm{LOO}}_k \;\triangleq\; \mathcal{I}(Y; C_k \mid x, C_{\setminus k}) = \mathcal{H}(Y \mid x, C_{\setminus k}) - \mathcal{H}(Y \mid x, C),
\end{equation}
which quantifies the information about $Y$ contributed by clarifying $s_k$ after all remaining spans have been resolved.
\end{definition}\vspace{-1.5mm}
Here, the second equality expresses the attribution as the difference between two conditional entropies. This form is particularly useful for estimation, as it allows us to approximate $\phi^{\mathrm{LOO}}_k$ by sampling model outputs under partially clarified (i.e., conditioned on $C_{\setminus k}$) and fully clarified (i.e., conditioned on $C$) inputs. We provide the derivation in Appendix~\ref{sec:appendix_proofs}.

\textbf{Challenges.}~~
While Definition \ref{def:loo} provides a principled per-span attribution, it evaluates each span only in the setting where all other spans have already been resolved. When spans are semantically dependent, conditioning on $C_{\setminus k}$ can implicitly constrain the plausible interpretations of $s_k$. In our running example, clarifying \textit{``the man in the suit''} as a bank robber effectively eliminates the riverbank interpretation of \textit{``the bank''}, thereby reducing the uncertainty originally associated with that span. Such dependencies hinder accurate measurement of how much each span contributes to the total aleatoric uncertainty, and thus the sum of LOO attributions does not, in general, equal the total aleatoric uncertainty $\mathcal{I}(Y; C \mid x)$. To address this limitation, we introduce a \textit{cooperative game formulation} based on Shapley values~\cite{shapleyvalue_shapley_american}. This approach accounts for the marginal contribution of each span across all possible subsets (coalitions) of clarifications, providing a fair allocation of uncertainty that satisfies the efficiency axiom~\cite{interpretableML_molnar}.
\vspace{-0.5em}

\section{The \proposed Framework} \label{sec:method}

\subsection{Uncertainty Localization via Shapley Values}
We formulate uncertainty localization as a cooperative game in which each ambiguous span $s_k \in \mathcal{S}$ acts as a player. The goal is to fairly distribute the total ``payoff'' -- defined as the total aleatoric uncertainty $\mathcal{I}(Y; C \mid x)$ -- across these spans. Let $\mathcal{N} = \{1, \ldots, n\}$ be the set of indices corresponding to the spans in $\mathcal{S}$. We specify the game by defining a value function over coalitions of spans $\mathcal{A} \subseteq \mathcal{N}$.

\begin{definition}[\textbf{Value Function}]\label{def:value_function}
The value function $v: 2^N \to \mathbb{R}$ assigns to each coalition $\mathcal{A} \subseteq \mathcal{N}$ the mutual information between $Y$ and the joint clarifications of the spans indexed by $\mathcal{A}$:
\begin{equation} \label{eq:value_function}
    v(\mathcal{A}) = \mathcal{I}(Y; C_\mathcal{A} \mid x) = \mathcal{H}(Y \mid x) - \mathbb{E}_{c_\mathcal{A} \sim q(C_{\mathcal{A}} \mid x)}\bigl[\mathcal{H}(Y \mid x, c_\mathcal{A})\bigr],
\end{equation}
where $C_\mathcal{A} = \big(C_j : j \in \mathcal{A} \big)$ denotes the joint clarifications associated with the spans in $\mathcal{A}$. By convention, we have $v(\emptyset) = 0$ and $v(\mathcal{N}) = \mathcal{I}(Y; C \mid x)$.
\end{definition}

Based on the value function in \eqref{eq:value_function}, we define the \textit{uncertainty attribution} for each span $s_k$ using the Shapley value $\phi_k$, which is defined as the weighted average of its marginal contributions $v(\mathcal{A} \cup \{k\}) - v(\mathcal{A})$ over all possible coalitions $\mathcal{A} \subseteq \mathcal{N} \setminus \{k\}$~\citep{SHAP_lundberg_nips}. The following theorem expresses these marginal contributions in terms of conditional mutual information:

\begin{theorem}[\textbf{Shapley Value of Ambiguous Span}]\label{thm:shapley}
For any coalition $\mathcal{A} \subseteq \mathcal{N} \setminus \{k\}$, the marginal contribution of span $s_k$ satisfies
\begin{equation}\label{eq:marginal}
    v(\mathcal{A} \cup \{k\}) - v(\mathcal{A}) \;=\; \mathcal{I}(Y; C_k \mid x, C_{\mathcal{A}}).
\end{equation}
Consequently, the Shapley value of span $s_k$ can be written as
\begin{equation}\label{eq:shapley_mi}
    \phi_k \;=\; \sum_{\mathcal{A} \subseteq \mathcal{N} \setminus \{k\}}
    \frac{|\mathcal{A}|!(n-|\mathcal{A}|-1)!}{n!}\,
    \mathcal{I}(Y; C_k \mid x, C_\mathcal{A}), \vspace{-1mm}
\end{equation}
which characterizes $\phi_k$ as the expected information gain regarding $Y$ from clarifying $s_k$, marginalized over all possible clarification contexts $C_{\mathcal{A}}$. 
\end{theorem}\vspace{-1.5mm}

The derivation of \eqref{eq:marginal} is provided in Appendix~\ref{sec:appendix_proofs}. 
Notably, \eqref{eq:shapley_mi} reveals that the LOO attribution $\phi^{\mathrm{LOO}}_k$ (Definition~\ref{def:loo}) is a specific boundary case of the Shapley value, where the context is restricted to the largest possible coalition $\mathcal{A} = \mathcal{N}\setminus \{k\}$. While LOO evaluates the contribution of a span only in the fully clarified setting, the Shapley formulation accounts for dependencies among spans by averaging their contributions across all clarification contexts. This leads to the following structural result:

\begin{remark}[Efficiency]\label{rem:efficiency}
The Shapley values in \eqref{eq:shapley_mi} partition the total input-induced uncertainty:
\begin{equation} \label{eq:efficiency}
    \sum_{k=1}^{n} \phi_k \;=\; v(\mathcal{N}) \;=\; \mathcal{I}(Y; C \mid x).
\end{equation}
\end{remark}\vspace{-1.5mm}

Remark~\ref{rem:efficiency} shows that our framework satisfies the efficiency axiom~\citep{shapleyvalue_shapley_american,interpretableML_molnar}, ensuring that the sum of \textit{local} span-level attributions exactly recovers the \textit{global} input-level aleatoric uncertainty. Together with Theorem~\ref{thm:shapley}, this establishes our approach as a \textit{generalization} of prior input-level uncertainty estimation~\citep{ICE_hou_arxiv}. In particular, the global quantity $\mathcal{I}(Y; C \mid x)$ introduced in~\citep{ICE_hou_arxiv} is recovered via efficiency as $\sum_k \phi_k = v(\mathcal{N})$, while our Shapley formulation further provides a principled per-span decomposition that localizes the underlying sources of the total aleatoric uncertainty.

\subsection{Practical Estimation of Uncertainty Attribution}
\begin{figure}[!t]
    \centering
    \includegraphics[width=0.95\linewidth]{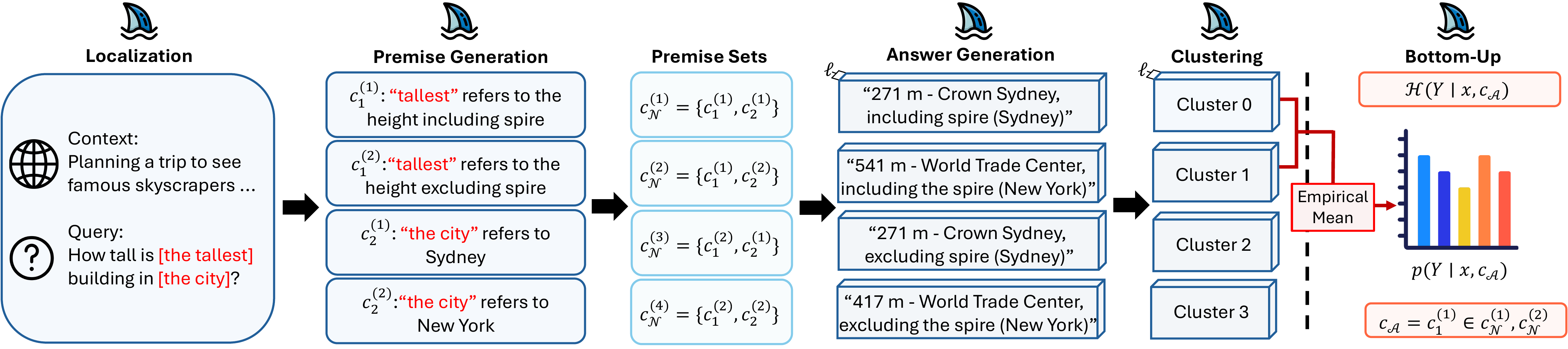} \vspace{-1mm}
    \caption{\textbf{An overview of the \proposed pipeline.} Given an input query \textit{``How tall is the tallest building in the city?''}, \textbf{(Step 1)} two ambiguous spans are localized, and then two premises are generated per span, providing four joint clarifications. For each clarification, \textbf{(Step 2)} the \texttt{Answerer} generates $\ell$ responses, which are then mapped into a shared semantic space by the \texttt{Clusterer}. \textbf{(Step 3)} Bottom-up marginalization is used to compute $\mathcal{H}(Y \mid x, c_\mathcal{A})$ for each premise (here, we set $c_\mathcal{A} = c_1^{(1)}$).}
    \label{fig:shaq}
    \vspace{-2.0mm}
\end{figure}
Estimating the Shapley values defined in \eqref{eq:shapley_mi} reduces to computing the expected conditional entropy $\mathbb{E}_{c_\mathcal{A} \sim q(C_{\mathcal{A}} \mid x)}
[\mathcal{H}(Y \mid x, c_{\mathcal{A}})]$ for all subsets $\mathcal{A} \subseteq \mathcal{N}$.
Our framework employs four specialized LLMs, each assigned a distinct role to facilitate a modular estimation process as in Figure~\ref{fig:shaq}: \vspace{-1mm}
\begin{itemize}[leftmargin=1.5em, itemsep=0pt, topsep=1pt]
    \item \texttt{Localizer}: Extracts the set of ambiguous spans $\mathcal{S}$ within the input $x$.
    \item \texttt{Generator}: Samples logical premises (i.e., clarifications) for each identified ambiguous span.
    \item \texttt{Answerer}: Generates responses $Y$ conditioned on $x$ and joint clarifications for a coalition $\mathcal{A}$.
    \item \texttt{Clusterer}: Maps diverse responses into a unified semantic space for entropy calculation. 
\end{itemize}

The estimation pipeline proceeds through the following three stages: \vspace{-1mm}
\begin{itemize}[leftmargin=1.0em, itemsep=0pt, topsep=1pt]
\item \textbf{Stage 1: Span Localization and Premise Generation.}~
Given an input $x$, we identify a collection of ambiguous spans $\mathcal{S} = \{s_1, \ldots, s_n\}$ using the \texttt{Localizer}. This step reduces the combinatorial space of ambiguity by focusing on a small number of semantically coherent units that can have multiple plausible interpretations. Then, for each identified span $s_k \in \mathcal{S}$, the \texttt{Generator} samples $m$ conditioning premises $\{c_k^{(1)}, \ldots, c_k^{(m)}\}$, where each premise represents a distinct \textit{logical clarification} of that span. We use the term \textit{premise} to emphasize that these clarifications act as grounding statements: by specifying a plausible interpretation of an ambiguous span, a premise defines a concrete semantic context that can alter the meaning of the input.

\item \textbf{Stage 2. Answer Sampling and Semantic Clustering.}~
For each $c_\mathcal{N}$, the \texttt{Answerer} generates $\ell$ responses conditioned on $x$ and $c_\mathcal{N}$. This produces a global pool of answers across all possible joint clarifications. To maintain consistency across subset evaluations, we adopt a \textit{global clustering} strategy, in which the \texttt{Clusterer} performs a single, comprehensive clustering over the entire answer pool. This defines a shared semantic output space $\mathcal{Y} = \{y_1, \dots, y_K\}$, from which the empirical distribution $p(Y \mid x, c_\mathcal{N})$ is constructed. Using a common output space ensures that all entropy estimates are directly comparable across coalitions, avoiding the incomparability that would arise from independently constructed semantic clusters.

\item \textbf{Stage 3. Hierarchical Conditional Entropy Computation.}~
Computing the Shapley value requires estimating $v(\mathcal{A})=\mathbb{E}_{c_{\mathcal{A}}\sim q(C_{\mathcal{A}}\mid x)}[\mathcal{H}(Y\mid x,c_{\mathcal{A}})]$ for all coalitions $\mathcal{A} \subseteq \mathcal{N}$. A naive approach would estimate each coalition independently via separate LLM calls; however, such an approach generally fails to preserve the required marginalization consistency, $p(Y \mid x, c_\mathcal{A}) = \sum_{c_k} q(c_k \mid x)\, p(Y \mid x, c_\mathcal{A}, c_k)$ where $k \notin \mathcal{A}$. As a consequence, the empirical marginal contribution $\mathcal{I}(Y; C_k \mid x, C_\mathcal{A})$ can have negative values due to an artifact of inconsistent estimation, thereby undermining the reliability of the resulting uncertainty attributions. To resolve this issue, we introduce a \emph{bottom-up marginalization} algorithm that treats unclarified spans in any coalition $\mathcal{A} \subset \mathcal{N}$ as marginalized variables.\footnote{This approach is structurally analogous to the interventional (marginal) Shapley value in feature attribution~\citep{SHAP_lundberg_nips,shapAlgorithm_suinlee_nature,truetotheModelorData_suinlee_arxiv}, where ``absent'' features are integrated out using their marginal distributions.} 
The independence of our premise generation ensures that the premises of unclarified spans remain independent of those already specified, i.e., $q(c_{\setminus \mathcal{A}} \mid x, c_\mathcal{A}) = q(c_{\setminus \mathcal{A}} \mid x)$. This justifies deriving $p(Y\mid x, c_{\mathcal{A}})$ for any coalition by marginalizing over the remaining premises. Crucially, since the premises $c_{\setminus\mathcal{A}}$ are readily available from \textbf{Stage 1}, $p(Y\mid x, c_{\mathcal{A}})$ can be obtained via Monte Carlo approximation over the existing sampled premises, without additional sampling.
\end{itemize}

This bottom-up marginalization guarantees a fundamental structural property required for stable information-theoretic uncertainty quantification:
\begin{property}[Hierarchical Monotonicity]\label{property:monotonicity}
For any subset $\mathcal{A} \subseteq \mathcal{N}$ and any span $k \notin \mathcal{A}$, the expected conditional entropy decreases monotonically as more spans are clarified:
\begin{equation}\label{eq:monotonicity}
    \mathbb{E}_{c_\mathcal{A}}[\mathcal{H}(Y \mid x, c_\mathcal{A})]
    \;\geq\;
    \mathbb{E}_{c_{\mathcal{A} \cup \{k\}}}[\mathcal{H}(Y \mid x, c_{\mathcal{A} \cup \{k\}})].
\end{equation}
Equivalently, every marginal contribution of the ambiguous span in \eqref{eq:shapley_mi} is non-negative:
\begin{equation}\label{eq:nonneg}
    \mathcal{I}(Y; C_k \mid x, C_\mathcal{A})=
    \mathbb{E}_{c_\mathcal{A}}[\mathcal{H}(Y \mid x, c_\mathcal{A})]
    - \mathbb{E}_{c_{\mathcal{A} \cup \{k\}}}[\mathcal{H}(Y \mid x, c_{\mathcal{A} \cup \{k\}})] \ge 0.
\end{equation}
The derivation is provided in Appendix~\ref{sec:appendix_proofs}.
\end{property}
Algorithm~\ref{alg:bottomup} formalizes this procedure by computing $\mathbb{E}_{c_\mathcal{A}}[\mathcal{H}(Y \mid x, c_\mathcal{A})]$ for all coalitions through marginalization over unclarified spans. These quantities are then used to evaluate \eqref{eq:shapley_mi}.
\refstepcounter{algorithm}\label{alg:bottomup}
\begin{tcolorbox}[
  enhanced,
  colback=gray!8,        
  colframe=gray!40,      
  leftrule=3pt, rightrule=0pt, toprule=0pt, bottomrule=0pt,
  arc=0pt,
  left=8pt,
  fontupper=\small
]

{\textsc{\textbf{Algorithm \thealgorithm: Bottom-up Marginalization}}}\medskip\\
\textbf{Input:} Bottom-level (finest) samples from $\{p(Y \mid x, c_\mathcal{N})\}$ for all joint clarifications.\\
\textbf{For each} subset $\mathcal{A} \subseteq \mathcal{N}$ and each premise combination $c_\mathcal{A}$\textbf{:}
\begin{enumerate}
    \item Identify all bottom-level $c_\mathcal{N}$ consistent with $c_\mathcal{A}$,
    i.e., $c_\mathcal{N} = c_\mathcal{A} \cup c_{\setminus \mathcal{A}}$ for some $c_{\setminus \mathcal{A}}$.
    \item Compute
    $p(Y \mid x, c_\mathcal{A}) = 
    \sum_{c_{\setminus \mathcal{A}}} q(c_{\setminus\mathcal{A}}\mid x)\,p(Y \mid x, c_\mathcal{A}\cup c_{\setminus \mathcal{A}})$.
    \item Compute
    $\mathcal{H}(Y \mid x, c_\mathcal{A}) = -\sum_{y} p(y \mid x, c_\mathcal{A}) \log p(y \mid x, c_\mathcal{A})$.
\end{enumerate}
\textbf{For each} subset $\mathcal{A} \subseteq \mathcal{N}$\textbf{:}
\begin{enumerate}
    \setcounter{enumi}{3}
    \item Compute
    $\mathbb{E}_{c_\mathcal{A}}[\mathcal{H}(Y \mid x, c_\mathcal{A})]
    = \sum_{c_\mathcal{A}} q(c_{\mathcal{A}}\mid x)\, \mathcal{H}(Y \mid x, c_\mathcal{A})$.
\end{enumerate}
\textbf{Output:} Cached values $\mathbb{E}_{c_\mathcal{A}}[\mathcal{H}(Y \mid x, c_\mathcal{A})]$ for Shapley
computation via \eqref{eq:shapley_mi}.
\end{tcolorbox}

\section{Experiments}\label{sec:experiments}

\begin{figure*}[t]
\centering
\begin{minipage}[t]{0.8\linewidth}
\centering
\captionof{table}{Uncertainty-based ambiguity detection.}\vspace{-1.5mm}
\label{tab:main_table}
{\renewcommand{\arraystretch}{0.9}
\resizebox{\linewidth}{!}{%

\begin{tabular}{p{4.0cm}|ccc|ccc}
\toprule
\rowcolor{gray!8}
 & \multicolumn{3}{c|}{AmbigQA (GPT-4)} & \multicolumn{3}{c}{AmbiEnt (GPT-5.4-mini)}\\[4pt]
\rowcolor{gray!8}
\textit{Method} & \textbf{F1 Score} & \textbf{AUROC} & \textbf{AUPRC} & \textbf{F1 Score} & \textbf{AUROC} & \textbf{AUPRC} \\
\midrule
\multicolumn{7}{l}{\textit{Other Uncertainty Methods}} \\[2pt]
Semantic Entropy  & 44.58 & 57.16 & 57.00 & 42.20 & 57.56 & 54.36 \\
Sample Diversity  & 55.24 & 60.86 & 59.35 & 48.33 & 59.23 & 57.48 \\
Sample Repetition  & 56.98 & 60.90 & 59.24 & 43.69 & 56.67 & 57.78 \\
Self-Consistency  & 53.40 & 59.33 & 57.59 & 47.93 & 59.28 & 59.46 \\
\cmidrule(lr){1-7}
\multicolumn{7}{l}{\textit{Aleatoric Methods}} \\[2pt]
Ask4Conf-D  & 56.36 & 52.44 & 53.97 & 65.43 & 50.82 & 52.54 \\
Deep Ensembles  & 59.70 & 60.29 & \underline{60.31} & 56.37 & 58.96 & 61.49 \\
ICE  & 65.54 & 60.66 & 59.98 & 57.33 & 55.40 & 53.21 \\
LOO  & 67.77 & 50.43 & 54.35  & 66.07 & 49.04 & 57.00 \\
LOO (Max Span)  & 67.77 & 52.36 & 55.80 & 66.07 & 52.66 & 59.69 \\
\midrule
\rowcolor{blue!8}
\textbf{\proposed(Ours, Max Span)}  & \textbf{70.38} & \underline{64.20} & 59.95  & \textbf{74.60} & \textbf{77.80} & \underline{75.19} \\
\rowcolor{blue!8}
\textbf{\proposed(Ours)}  & \textbf{70.38} & \textbf{65.77} &\textbf{61.52}  & \underline{73.84} & \underline{77.68} & \textbf{75.71} \\
\bottomrule 
\end{tabular} \vspace{-2mm}
}}
\vspace{-4mm}
\end{minipage}
\end{figure*}

\begin{figure*}[t]
\hfill
\begin{minipage}[t]{0.48\linewidth}
\centering
\captionof{table}{\texttt{Answerer \& Clusterer} backbone with \texttt{GPT-5.4-mini} / \texttt{Gemini models}} \vspace{-1.5mm}
\label{tab:main_table2}
\vspace{0.05em}

{\renewcommand{\arraystretch}{0.85}
\resizebox{\linewidth}{!}{%
    
\begin{tabular}{lccc}
\toprule
\rowcolor{gray!8}
\textit{Method} & \textbf{F1 Score} & \textbf{AUROC} & \textbf{AUPRC} \\
\midrule
\multicolumn{4}{l}{\texttt{GPT-5.4-mini}} \\
Deep Ensembles  & 63.41 & 62.62 & 57.29 \\
ICE  & 66.67 & \textbf{64.61} & \textbf{63.11}\\
\midrule
\rowcolor{blue!6}
\textbf{\proposed(Ours, Max Span)}  & \textbf{69.53} & 62.37 & 58.63 \\
\rowcolor{blue!8}
\textbf{\proposed(Ours)}  & \textbf{69.53} & \underline{63.60} & \underline{62.08} \\
\midrule
\multicolumn{4}{l}{\texttt{Gemini-3.1-flash-lite-preview}} \\
Deep Ensembles  & 44.30 & 55.92 & 54.26 \\
ICE  & 63.21 & \underline{63.41} & \textbf{63.53}\\
\midrule
\rowcolor{blue!6}
\textbf{\proposed(Ours, Max Span)}  & \textbf{66.67} & 62.97 & 58.95 \\
\rowcolor{blue!8}
\textbf{\proposed(Ours)}  & \textbf{66.67} & \textbf{64.22} & \underline{63.17} \\

\midrule
\multicolumn{4}{l}{\texttt{Gemini-2.5-flash}} \\
Deep Ensembles  & 56.28 & 55.42 & 53.42 \\
ICE  & 65.82 & 61.71 & \underline{61.15}\\
\midrule
\rowcolor{blue!6}
\textbf{\proposed(Ours, Max Span)}  & \textbf{68.75} & \underline{63.50} & 60.06 \\
\rowcolor{blue!8}
\textbf{\proposed(Ours)}  & \textbf{68.75} & \textbf{64.84} &\textbf{62.74} \\
\bottomrule
\end{tabular}

}
}
\end{minipage}
\hfill
\begin{minipage}[t]{0.48\linewidth}
\centering
\captionof{table}{\texttt{Answerer \& Clusterer} backbone with \texttt{Gemini models}} \vspace{-1.5mm}
\label{tab:main_table2_NLI}
\vspace{0.05em}
{\renewcommand{\arraystretch}{0.85}
\resizebox{\linewidth}{!}{%
    \begin{tabular}{lccc}
\toprule
\rowcolor{gray!8}
\textit{Method} & \textbf{F1 Score} & \textbf{AUROC} & \textbf{AUPRC} \\
\midrule

\multicolumn{4}{l}{\texttt{Gemini-3.1-flash-lite-preview}} \\
Deep Ensembles  & 27.52 & 46.74 & 47.76 \\
ICE  & 50.70 & 50.71 & 49.45 \\
Ask4Conf-D  & 49.12 & 45.51 & 49.57 \\
LOO  & \textbf{66.07} & 53.23 & 55.63 \\
LOO (Max Span)  & 65.47 & 52.94 & 56.20 \\
\midrule
\rowcolor{blue!6}
\textbf{\proposed(Ours, Max Span)}  & \underline{65.59} & \textbf{72.39} & \textbf{70.65} \\
\rowcolor{blue!8}
\textbf{\proposed(Ours)}  & \underline{65.59} & \underline{71.96} & \underline{69.32} \\

\midrule
\multicolumn{4}{l}{\texttt{Gemini-2.5-flash}} \\
Deep Ensembles  & 41.66 & 54.28 & 56.19 \\
ICE  & 51.38 & 51.31 & 49.97 \\
Ask4Conf-D  & 60.20 & 42.86 & 46.83 \\
LOO  & 66.07 & 38.64 & 44.39 \\
LOO (Max Span)  & 65.47 & 41.48 & 46.22 \\
\midrule
\rowcolor{blue!6}
\textbf{\proposed(Ours, Max Span)}  & \textbf{71.42} & \underline{76.40} & \underline{74.52} \\
\rowcolor{blue!8}
\textbf{\proposed(Ours)}  & \textbf{71.42} & \textbf{76.47} & \textbf{75.78} \\

\bottomrule
\end{tabular}
}
}
\end{minipage}

\vspace{3mm}
\centering
\includegraphics[width=1\linewidth]{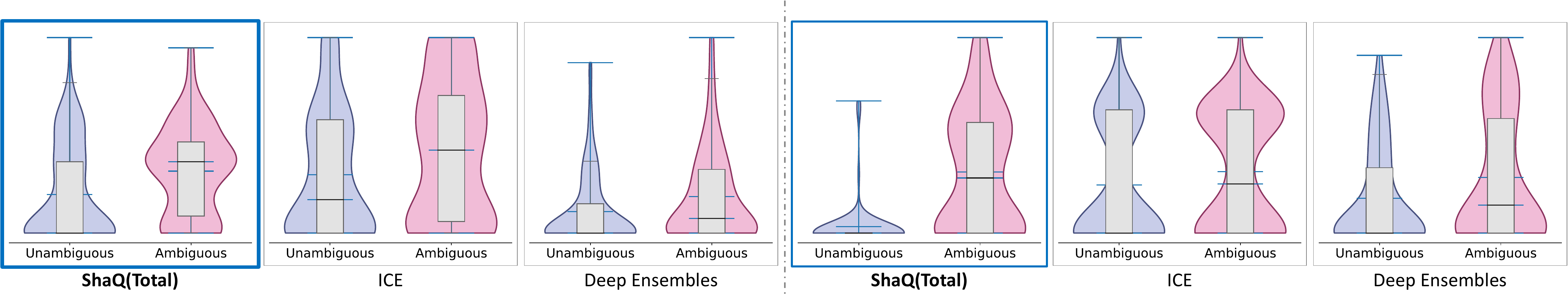} \vspace{-5mm}
\captionof{figure}{\textbf{Aleatoric uncertainty distributions on AmbigQA (Left) and AmbiEnt (Right).}} 
\vspace{-5.0mm}
\label{fig:Violin}
\end{figure*}

\subsection{Experiment setup}
\vspace{-1.5mm}
\textbf{Baselines.}~~
For baseline methods that explicitly estimate aleatoric uncertainty, we use Input Clarification Ensemble (ICE)~\citep{ICE_hou_arxiv}, Deep Ensembles~\citep{deepensemble_lakshiminarayana_nips}, and Ask4Conf-D~\citep{ICE_hou_arxiv,ask4conf_tian_emnlp}. We also evaluate answer-confidence-based estimators such as Semantic Entropy~\citep{SemanticUncertainty_kuhn_iclr}, Sample Diversity~\citep{si2022prompting}, Sample Repetition~\citep{cole2023selectively}, and Self-Consistency~\citep{cole2023selectively}, following prior work~\citep{ICE_hou_arxiv}. 

\textbf{Implementation Details.}~~
For AmbigQA~\citep{ambigqa_sewon_emnlp}, we use \texttt{GPT-4} as the primary backbone for all baselines and modules of \proposed, following the clarifier backbone used in ICE~\citep{ICE_hou_arxiv}. 
For AmbiEnt~\citep{ambient-liu-etal-2023-afraid}, we use \texttt{GPT-5.4 mini}, as its natural language inference (NLI) formulation requires stronger reasoning capabilities~\citep{morishita2024enhancing}. 
To examine whether the observed gains are robust to the choice of LLM, we additionally evaluate \proposed and the aleatoric uncertainty baselines using \texttt{Gemini}-based models for ambiguity detection.
Following prior work~\citep{ICE_hou_arxiv}, \textit{Deep Ensembles} are implemented by varying the in-context examples across ensemble members, since parameter-level ensembling is infeasible for frozen LLMs. 
We further apply semantic clustering to both ICE and \textit{Deep Ensembles} to ensure a fair comparison. 
Full configurations and prompts are provided in Appendix~\ref{sec:Configurations}.

\subsection{Ambiguity Detection with Aleatoric Uncertainty}\label{subsec:ambiguity_detection}
To assess whether the aleatoric uncertainty estimated by \proposed aligns with true input ambiguity, we follow the ambiguity-detection protocol of prior work~\citep{ICE_hou_arxiv} and treat the aleatoric uncertainty produced by each method as signals for ambiguity detection. An effective aleatoric uncertainty measure should assign higher uncertainty to ambiguous questions and lower uncertainty to unambiguous ones. 
To this end, we evaluate two ambiguity benchmarks, AmbigQA~\citep{ambigqa_sewon_emnlp} and AmbiEnt~\citep{ambient-liu-etal-2023-afraid}, measuring how well each method distinguishes ambiguous from unambiguous questions.

Since \proposed produces span-level uncertainty attributions ($\phi_k$), we use~\eqref{eq:efficiency}, i.e., $\sum_{k \in [n]} \phi_k$, as the total aleatoric uncertainty for ambiguity detection, and denote it by \proposed. Additionally, we use the maximum span-level uncertainty attribution, $\max_{k \in [n]} \phi_k$, denoted by \proposed~(\texttt{Max}), as a direct variant. To examine the role of interaction modeling, we include LOO-based variants~\eqref{eq:span_loo} as direct ablations, assessing whether evaluating a span's contribution under a fully-clarified context is sufficient in such settings: LOO (i.e., $\sum_{k=1}^{n} \phi^{\mathrm{LOO}}_k$) and LOO-Max (i.e., $\max_{k \in [n]} \phi^{\mathrm{LOO}}_{k}$).

\textbf{Ambiguity Detection.}~~
We evaluate ambiguity detection performance using the best F1 score, AUROC, and AUPRC. As shown in Tables~\ref{tab:main_table}, ~\ref{tab:main_table2} and \ref{tab:main_table2_NLI}, \proposed and \proposed(\texttt{Max}) consistently achieve strong performance across different backbones and evaluation metrics, demonstrating the robustness of our uncertainty localization framework. In particular, on AmbiEnt, where ambiguity often involves more complex and compositional phenomena, \proposed and \proposed~(\texttt{Max}) outperform existing baselines by a substantial margin in AUROC and AUPRC. This indicates that \proposed is especially effective at quantifying the degree of ambiguity, going \textit{beyond simply detecting its presence}.

Furthermore, Figure~\ref{fig:Violin} shows the distributions of aleatoric uncertainty for ambiguous and unambiguous inputs. Compared to the baselines, \proposed exhibits the clearest separation between ambiguous and unambiguous inputs, assigning consistently higher uncertainty to ambiguous questions. This separation supports the effectiveness of \proposed as a fine-grained ambiguity quantification method. 

Overall, these results show that the uncertainty estimates produced by \proposed closely align with human-annotated ambiguity, while the gains over the aggregate input-level baseline (i.e., ICE) highlight the benefit of explicitly localizing ambiguous spans.


\begin{wraptable}{r}{0.48\textwidth}
\vspace{-1.5em}
\centering
\caption{Multi-span results on AmbiEnt}
\label{tab:many_span_main}
\resizebox{\linewidth}{!}{%
\begin{tabular}{lccc}
\toprule
\rowcolor{gray!8}
\textit{Method} & \textbf{F1 Score} & \textbf{AUROC} & \textbf{AUPRC} \\
\midrule
Ask4Conf-D  & 83.95 & 40.40 & 78.13 \\
Deep Ensembles  & 58.62 & 45.95 & 75.25 \\
ICE  & 59.64 & 53.40 & 77.54 \\
LOO  & 85.36 & 49.11 & 79.59 \\
LOO(Max Span)  & 85.36 & 51.64 & 80.26 \\
\midrule
\rowcolor{blue!6}
\textbf{\proposed(Ours, Max)}  & \textbf{86.95} & \textbf{81.31} & \textbf{92.39} \\
\rowcolor{blue!8}
\textbf{\proposed(Ours)}  & \underline{86.56} & \underline{80.05} & \underline{91.27} \\
\bottomrule
\end{tabular}%
}
\vspace{-1.0em}
\end{wraptable}

\textbf{Multi Span Ambiguity Detection.}~~  
We further analyze performance on AmbiEnt inputs where our localization stage identifies two or more candidate ambiguous spans (i.e., $|\mathcal{S}| \geq 2$). This case provides a more rigorous evaluation of localization, as multiple sources of ambiguity may coexist and interact within a single input. 


In Table~\ref{tab:many_span_main}, \proposed and its variants achieve the best performance, outperforming all baselines and ablations, including LOO and LOO-Max. These results support the central motivation of \proposed: in the presence of interacting ambiguity sources, aleatoric uncertainty is more accurately attributed through a dependency-aware decomposition, as opposed to a static, fully-clarified baseline.
Appendix~\ref{sec:Appendix Multi Span} details further AmbigQA multi-span results and highlights \proposed's inherent robustness to imperfect localization. Specifically, our formulation naturally neutralizes extraction errors by assigning negligible Shapley values to falsely identified spans.

\subsection{Uncertainty-Guided Clarification} \label{subsec:uncertainty_guidance}
\textbf{Clarification Simulation Setup.}~~
We evaluate whether each method provides \textit{actionable uncertainty feedback} in an uncertainty-guided clarification setting. 
As illustrated in Figure~\ref{fig:Dialog}, we \textit{simulate a user-facing workflow} in which a method provides uncertainty feedback for an ambiguous input query, and an LLM user revises the original input $x$ into a clarified input $x'$. 
We apply this procedure to ambiguous questions from the AmbigQA benchmark~\citep{ambigqa_sewon_emnlp}. 
While baselines provide only input-level uncertainty feedback in the form of a single scalar score, \proposed provides localized span-level uncertainty attributions, allowing us to test whether localization supports more targeted clarification while preserving the original intent.
To avoid bias toward examples favored by a single estimator, we evaluate on a superset formed by the union of the top 20\% examples ranked by \proposed's maximum span-level aleatoric uncertainty and those ranked by the maximum aleatoric uncertainty among the baselines.
\begin{figure*}[t]
\centering
\vspace{-1em}

\begin{minipage}[t]{0.48\textwidth}
\centering
\includegraphics[width=\linewidth]{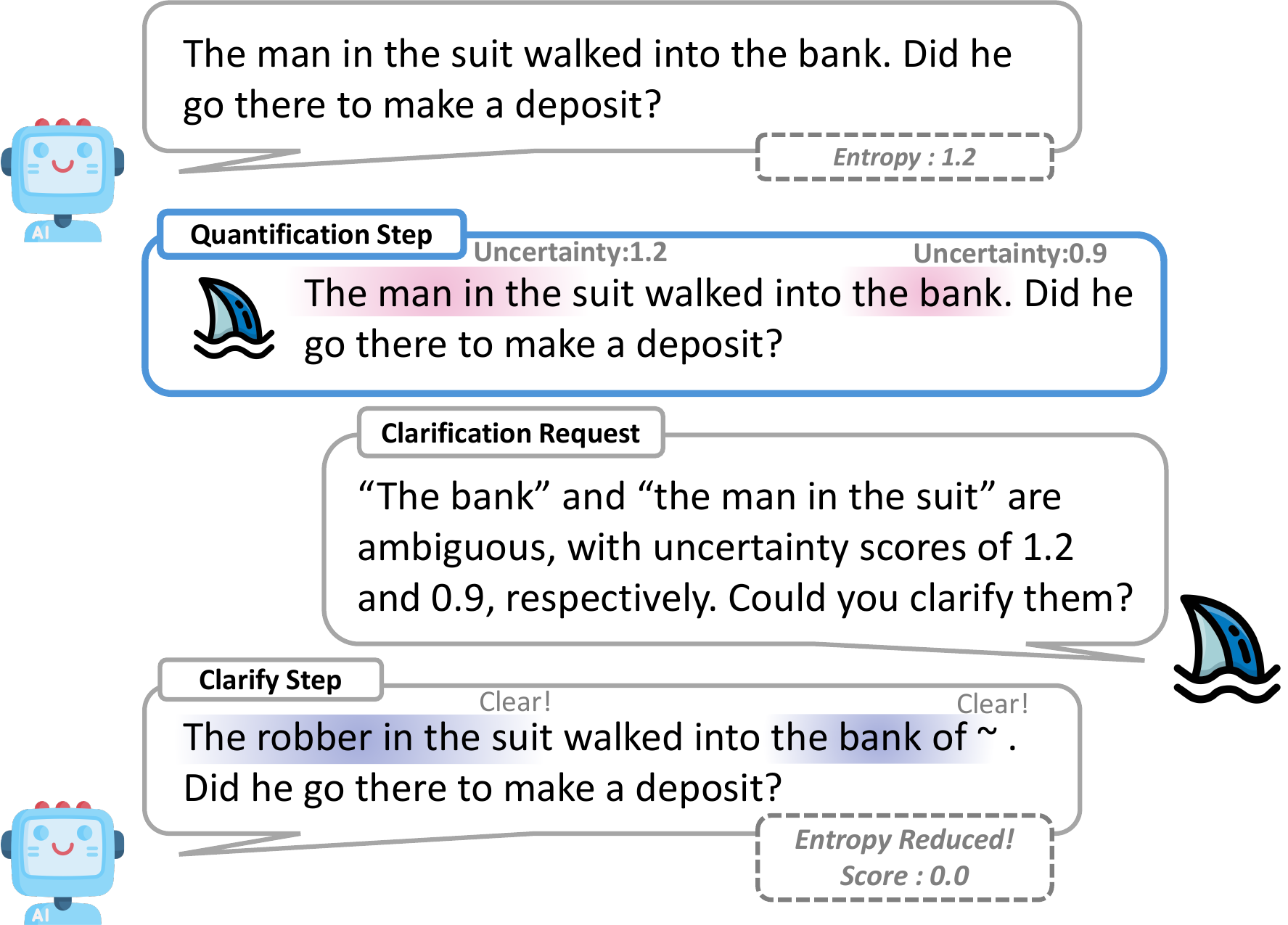}
\caption{\textbf{Uncertainty-guided clarification.}}
\label{fig:Dialog}
\end{minipage}
\hfill
\begin{minipage}[t]{0.48\textwidth}
\centering
\includegraphics[width=\linewidth]{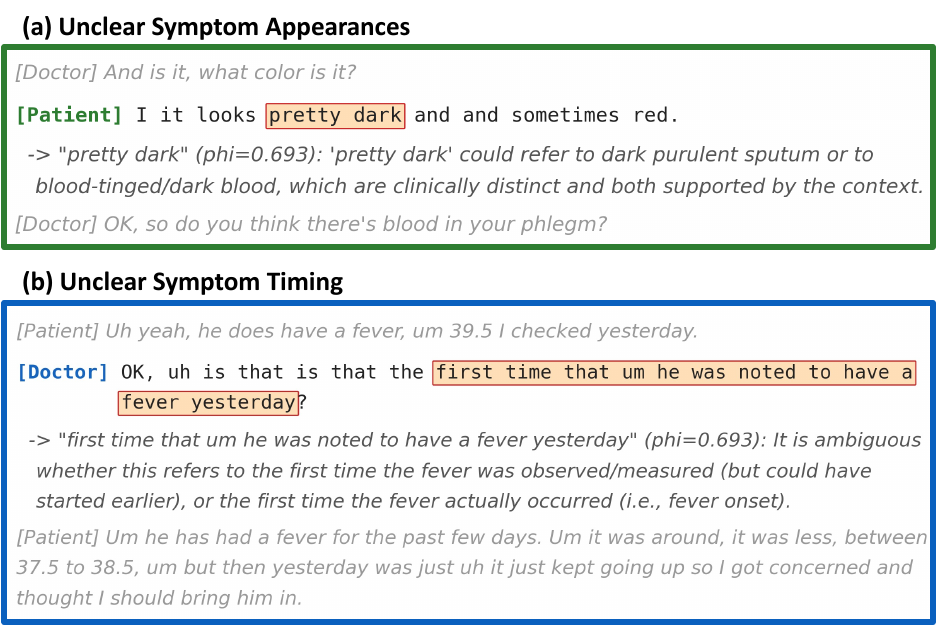}
\caption{\textbf{Qualitative results on MediTOD.}}
\label{fig:medical_dialogue}
\end{minipage}

\vspace{-2em} 
\end{figure*}

This captures both localized ambiguity and broadly high-uncertainty cases. 
We use the same backbone model as in Section~\ref{subsec:ambiguity_detection} (i.e., \texttt{GPT-4}); details and prompts are provided in Appendix~\ref{sec:Configurations}.

\textbf{Evaluation Metrics.}~~
To assess the efficiency of \proposed, we employ two complementary metrics that capture the trade-off between \textit{user effort and information gain}: entropy reduction, $\Delta \mathcal{H} = \mathcal{H}(Y|x) - \mathcal{H}(Y|x')$ and word-level edit distance (\textit{Edit}), computed as the Levenshtein distance over word tokens. These metrics should be interpreted jointly: high $\Delta \mathcal{H}$ indicates resolved model-side uncertainty, while low \textit{Edit} indicates minimal modification. The desired outcome is high entropy reduction with low editing effort, showing that uncertainty information enables targeted, ``surgical'' refinement of the ambiguity sources.

\begin{wraptable}{r}{0.48\textwidth}
\vspace{-10pt}
\centering
\caption{High Aleatoric Uncertainty Superset}\label{tab:clarification_superset}
\resizebox{\linewidth}{!}{%
\begin{tabular}{lccc}
\toprule
\rowcolor{gray!8}
\textit{Method} & Clarified & $\Delta H$($\uparrow$) & Edit ($\downarrow$)\\
\midrule
Semantic Entropy  &  0.3423 & 0.0330 & 6.4545 \\
ICE  &  0.4387 & -0.0633 & 6.7878\\
Deep Ensembles  &  0.3831 & -0.0076 & 7.4545\\
\midrule
\rowcolor{blue!8}
\textbf{\proposed(Ours)}&   \textbf{0.3365} & \textbf{0.0389} & \textbf{5.5151}\\
\bottomrule
\end{tabular}
}
\vspace{-10pt}
\end{wraptable}
\textbf{Overall Results.}~~
Table~\ref{tab:clarification_superset} reports the uncertainty-guided clarification results on the high-uncertainty superset.
We find that \proposed achieves the highest entropy reduction while requiring the lowest word-level edits, indicating that span-level uncertainty provides actionable guidance for resolving ambiguity. We also observe that baseline-guided revisions often require larger modifications and may even increase predictive entropy~(-0.0633), suggesting that input-level uncertainty alone can fail to identify which part of the query should be clarified. In Appendix~\ref{sec:Appendix Clarification}, we further analyze split evaluation sets under other selection criteria and observe the same trend.

\vspace{-0.5em}
To demonstrate that \proposed generalizes beyond synthetic LLM prompts to complex human-to-human interactions, we evaluate its performance on the MediTOD benchmark \citep{meditod_saley_emnlp}. This dataset consists of simulated doctor-patient interviews conducted between physicians (see Appendix~\ref{sec:Appendix Meditod} for details and additional results). Representative examples in Figure~\ref{fig:medical_dialogue} highlight the framework's ability to pinpoint clinically significant ambiguities. In panel~(a), the phrase \textit{''pretty dark''} is localized as ambiguous—potentially indicating either dark purulent sputum or a blood-tinged discharge. The doctor's immediate follow-up, inquiring about blood in the phlegm, perfectly validates this localization. In panel~(b), \textit{''first time that um he was noted to have a fever yesterday''} is flagged due to uncertainty between the time of observation and the actual onset of the fever. This is subsequently confirmed by the patient’s next turn, clarifying that the fever had actually been ongoing for several days. By providing targeted localization of underspecified information, \proposed shows significant promise for enhancing reliability in high-stakes domains.
\section{Conclusion}
\vspace{-0.7em}
\label{sec:conclusion}
In this paper, we addressed a limitation of existing UQ methods in LLM: the inability to identify the specific input sources driving aleatoric uncertainty. To bridge this gap, we proposed \proposed, a theoretically-grounded framework designed to localize input-induced uncertainty at the span level. By formulating ambiguity resolution as a cooperative game, \proposed models the semantic dependencies between ambiguous spans. Empirically, \proposed significantly improves ambiguity detection on benchmarks over monolithic baselines. Furthermore, it enables targeted input clarifications with minimal user effort and demonstrates potential for enhancing LLM reliability in high-stakes domains.

\bibliographystyle{icml2026}
\bibliography{arxiv_paper}

\clearpage
\appendix
\clearpage
{\Huge\bfseries Appendix\par}
\vspace{2em}
{\Large\bfseries Table of Contents\par}
\vspace{0.3em}
\hrule
\vspace{1em}

\noindent\textbf{A\quad Mathematical Derivations}\hfill\textbf{\pageref{sec:appendix_proofs}}\\[1.0em]

\noindent\textbf{B\quad Configurations and Prompts}\hfill\textbf{\pageref{sec:Configurations}}\\[0.5em]
\hspace*{1.5em} B.1\quad Dataset Description: AmbigQA\hfill\pageref{sec:ambigqa_desc}\\[0.3em]
\hspace*{1.5em} B.2\quad Dataset Description: AmbiEnt\hfill\pageref{sec:ambient_desc}\\[0.3em]
\hspace*{1.5em} B.3\quad Configurations\hfill\pageref{sec:conf}\\[0.3em]
\hspace*{1.5em} B.4\quad Prompts\hfill\pageref{sec:AmbigQA_Prompts}\\[0.3em]
\hspace*{1.5em} B.5\quad Prompts and Configurations for Uncertainty-guided Clarification\hfill\pageref{sec:Uncertainty-guided Clarification}\\[1.0em]

\noindent\textbf{C\quad Additional Results}\hfill\textbf{\pageref{sec:Additional Result}}\\[0.5em]
\hspace*{1.5em} C.1\quad Additional Results on Ambiguity Detection\hfill\pageref{sec:Appendix Ambiguity}\\[0.3em]
\hspace*{1.5em} C.2\quad Additional Ablation Results on Ambiguity Detection\hfill\pageref{sec:Appendix Ablation}\\[0.3em]
\hspace*{1.5em} C.3\quad Additional Results on Multi Span Ambiguity Detection\hfill\pageref{sec:Appendix Multi Span}\\[0.3em]
\hspace*{1.5em} C.4\quad Additional Results on Uncertainty-guided Clarification\hfill\pageref{sec:Appendix Clarification}\\[0.3em]
\hspace*{1.5em} C.5\quad Additional Results on Human-to-Human Medical Multi-turn Dialogue\hfill\pageref{sec:Appendix Meditod}\\[1.0em]

\noindent\textbf{D\quad Computational Analysis}\hfill\textbf{\pageref{sec:appendix_compute}}\\[1.0em]

\noindent\textbf{E\quad Limitations and Future Directions}\hfill\textbf{\pageref{sec:appendix_limitations}}\\[1.0em]

\noindent\textbf{F\quad Broader Impacts}\hfill\textbf{\pageref{sec:broader_impacts}}\\[1.0em]

\noindent\textbf{G\quad Asset Licenses and Terms of Use}\hfill\textbf{\pageref{Asset Licenses and Terms of Use}}\\[1.0em]

\vspace{0.5em}
\hrule
\clearpage

\section{Mathematical Derivations} \label{sec:appendix_proofs}
We restate Definition~\ref{def:loo}, Theorem~\ref{thm:shapley} and Property~\ref{property:monotonicity} for convenience.

\begin{definition}[\textbf{LOO Span Attribution}]
The LOO attribution of span $s_k$ is defined as the conditional mutual information
\begin{equation}
\phi^{\mathrm{LOO}}_k \;\triangleq\; \mathcal{I}(Y; C_k \mid x, C_{\setminus k}) = \mathcal{H}(Y \mid x, C_{\setminus k}) - \mathcal{H}(Y \mid x, C),
\end{equation}
which quantifies the information about $Y$ contributed by clarifying $s_k$ after all remaining spans have been resolved.
\end{definition}
\begin{proof}
By definition of conditional mutual information,
\begin{align}
\mathcal{I}(Y; C_k \mid x, C_{\setminus k}) 
&= \mathcal{H}(Y \mid x, C_{\setminus k}) - \mathcal{H}(Y \mid x, C_k, C_{\setminus k}).
\end{align}
Since $(C_k, C_{\setminus k})$ jointly constitute $C$, the second term satisfies
$\mathcal{H}(Y \mid x, C_k, C_{\setminus k}) = \mathcal{H}(Y \mid x, C)$,
which gives
\begin{align}
\mathcal{I}(Y; C_k \mid x, C_{\setminus k}) 
&= \mathcal{H}(Y \mid x, C_{\setminus k}) - \mathcal{H}(Y \mid x, C).
\end{align}
\end{proof}

\begin{theorem}[\textbf{Shapley Value of Ambiguous Span}]
For any coalition $\mathcal{A} \subseteq \mathcal{N} \setminus \{k\}$, the marginal contribution of span $s_k$ satisfies
\begin{equation}
v(\mathcal{A} \cup \{k\}) - v(\mathcal{A}) \;=\; \mathcal{I}(Y; C_k \mid x, C_{\mathcal{A}}).
\end{equation}
Consequently, the Shapley value of span $s_k$ can be written as
\begin{equation}
\phi_k \;=\; \sum_{\mathcal{A} \subseteq \mathcal{N} \setminus \{k\}}
\frac{|\mathcal{A}|!(n-|\mathcal{A}|-1)!}{n!}\,
\mathcal{I}(Y; C_k \mid x, C_\mathcal{A}).
\end{equation}
\end{theorem}
\begin{proof}
By Definition~\ref{def:value_function}, the marginal contribution expands as
\begin{align}
v(\mathcal{A} \cup \{k\}) - v(\mathcal{A})
&= \mathcal{I}(Y; C_{\mathcal{A} \cup \{k\}} \mid x) - \mathcal{I}(Y; C_{\mathcal{A}} \mid x) \\
&= \Bigl[\mathcal{H}(Y \mid x) - \mathbb{E}_{c_{\mathcal{A}\cup\{k\}}}\bigl[\mathcal{H}(Y \mid x, c_{\mathcal{A}\cup\{k\}})\bigr]\Bigr] \\
&\quad - \Bigl[\mathcal{H}(Y \mid x) - \mathbb{E}_{c_{\mathcal{A}}}\bigl[\mathcal{H}(Y \mid x, c_{\mathcal{A}})\bigr]\Bigr] \\
&= \mathbb{E}_{c_{\mathcal{A}}}\bigl[\mathcal{H}(Y \mid x, c_{\mathcal{A}})\bigr] 
   - \mathbb{E}_{c_{\mathcal{A}\cup\{k\}}}\bigl[\mathcal{H}(Y \mid x, c_{\mathcal{A}\cup\{k\}})\bigr].
\end{align}
By the definition of conditional mutual information and conditional entropy,
\begin{align}
\mathcal{I}(Y; C_k \mid x, C_{\mathcal{A}}) 
&= \mathcal{H}(Y \mid x, C_{\mathcal{A}}) - \mathcal{H}(Y \mid x, C_{\mathcal{A}}, C_k) \\
&= \mathbb{E}_{c_{\mathcal{A}}}\bigl[\mathcal{H}(Y \mid x, c_{\mathcal{A}})\bigr]
   - \mathbb{E}_{c_{\mathcal{A}\cup\{k\}}}\bigl[\mathcal{H}(Y \mid x, c_{\mathcal{A}\cup\{k\}})\bigr],
\end{align}
where the second equality uses $(C_{\mathcal{A}}, C_k) = C_{\mathcal{A} \cup \{k\}}$.
The two expressions are identical, which completes the proof.
\end{proof}

\begin{property}[\textbf{Hierarchical Monotonicity}]
For any subset $\mathcal{A} \subseteq \mathcal{N}$ and any span $k \notin \mathcal{A}$, the expected conditional entropy decreases monotonically as more spans are clarified:
\begin{equation}
    \mathbb{E}_{c_\mathcal{A}}[\mathcal{H}(Y \mid x, c_\mathcal{A})]
    \;\geq\;
    \mathbb{E}_{c_{\mathcal{A} \cup \{k\}}}[\mathcal{H}(Y \mid x, c_{\mathcal{A} \cup \{k\}})].
\end{equation}
Equivalently, every marginal contribution of the ambiguous span in Eq.~\eqref{eq:shapley_mi} is non-negative:
\begin{equation}
    \mathcal{I}(Y; C_k \mid x, C_\mathcal{A})=
    \mathbb{E}_{c_\mathcal{A}}[\mathcal{H}(Y \mid x, c_\mathcal{A})]
    - \mathbb{E}_{c_{\mathcal{A} \cup \{k\}}}[\mathcal{H}(Y \mid x, c_{\mathcal{A} \cup \{k\}})] \ge 0.
\end{equation}
\end{property}
\begin{proof}
By definition, the marginal contribution of span $s_k$ to coalition $\mathcal{A}$ is given by the difference in expected conditional entropy:
\begin{equation} \label{eq:proof_mi_def}
    \mathcal{I}(Y; C_k \mid x, C_\mathcal{A}) = \mathbb{E}_{c_\mathcal{A}}[\mathcal{H}(Y \mid x, c_\mathcal{A})] - \mathbb{E}_{c_{\mathcal{A} \cup \{k\}}}[\mathcal{H}(Y \mid x, c_{\mathcal{A} \cup \{k\}})].
\end{equation}
In our framework, the probability distribution for any coalition $\mathcal{A}$ is computed by marginalizing the completely clarified bottom-level states:
\begin{equation} \label{eq:proof_bottom_up}
    p(Y \mid x, c_\mathcal{A}) = \frac{1}{m^{|\mathcal{N} \setminus \mathcal{A}|}} \sum_{c_{\mathcal{N} \setminus \mathcal{A}}} p(Y \mid x, c_\mathcal{A} \cup c_{\mathcal{N} \setminus \mathcal{A}}).
\end{equation}
To relate $p(Y \mid x, c_\mathcal{A})$ to $p(Y \mid x, c_{\mathcal{A} \cup \{k\}})$, we partition the set of unclarified spans $\mathcal{N} \setminus \mathcal{A}$ into the single target span $\{k\}$ and the rest of the unclarified spans, which we denote as $\mathcal{N} \setminus \mathcal{A}'$ where $\mathcal{A}' = \mathcal{A} \cup \{k\}$.
Correspondingly, the combinations of premises can be decomposed as $c_{\mathcal{N} \setminus \mathcal{A}} = c_k \cup c_{\mathcal{N} \setminus \mathcal{A}'}$, and the total number of combinations factors as $m^{|\mathcal{N} \setminus \mathcal{A}|} = m \cdot m^{|\mathcal{N} \setminus \mathcal{A}'|}$.
Substituting this decomposition into Equation \eqref{eq:proof_bottom_up} allows us to rewrite the single summation as a nested summation:
\begin{align} \label{eq:proof_decomposition}
    p(Y \mid x, c_\mathcal{A}) &= \frac{1}{m \cdot m^{|\mathcal{N} \setminus \mathcal{A}'|}} \sum_{c_k} \sum_{c_{\mathcal{N} \setminus \mathcal{A}'}} p(Y \mid x, c_\mathcal{A} \cup c_k \cup c_{\mathcal{N} \setminus \mathcal{A}'}) \nonumber \\
    &= \frac{1}{m} \sum_{c_k} \underbrace{ \left( \frac{1}{m^{|\mathcal{N} \setminus \mathcal{A}'|}} \sum_{c_{\mathcal{N} \setminus \mathcal{A}'}} p(Y \mid x, c_{\mathcal{A}'} \cup c_{\mathcal{N} \setminus \mathcal{A}'}) \right) }_{\text{Definition of } p(Y \mid x, c_{\mathcal{A}'})}.
\end{align}
Observe that the term inside the parentheses is exactly the bottom-up definition for the subset $\mathcal{A}' = \mathcal{A} \cup \{k\}$. Therefore, we can recursively express $p(Y \mid x, c_\mathcal{A})$ as:
\begin{equation} \label{eq:proof_convex_comb}
    p(Y \mid x, c_\mathcal{A}) = \frac{1}{m} \sum_{c_k} p(Y \mid x, c_{\mathcal{A} \cup \{k\}}).
\end{equation}
Equation \eqref{eq:proof_convex_comb} mathematically establishes that $p(Y \mid x, c_\mathcal{A})$ is a strict convex combination of its immediate superset distributions $\{p(Y \mid x, c_{\mathcal{A} \cup \{k\}})\}_{c_k}$. Since the Shannon entropy $\mathcal{H}(\cdot)$ is a strictly concave function, applying Jensen's Inequality to this convex combination yields:
\begin{equation} \label{eq:proof_jensen}
    \mathcal{H}(Y \mid x, c_\mathcal{A}) \geq \frac{1}{m} \sum_{c_k} \mathcal{H}(Y \mid x, c_{\mathcal{A} \cup \{k\}}).
\end{equation}
Now, we take the expectation over $c_\mathcal{A} \sim q(c_\mathcal{A} \mid x)$ on both sides. Since our premises are sampled uniformly with probability $\frac{1}{m^{|\mathcal{A}|}}$, this gives:
\begin{align} \label{eq:proof_expectation}
    \mathbb{E}_{c_\mathcal{A}}[\mathcal{H}(Y \mid x, c_\mathcal{A})] &\geq \frac{1}{m^{|\mathcal{A}|}} \sum_{c_\mathcal{A}} \left( \frac{1}{m} \sum_{c_k} \mathcal{H}(Y \mid x, c_{\mathcal{A} \cup \{k\}}) \right) \nonumber \\
    &= \frac{1}{m^{|\mathcal{A}|+1}} \sum_{c_\mathcal{A} \cup c_k} \mathcal{H}(Y \mid x, c_{\mathcal{A} \cup \{k\}}) \nonumber \\
    &= \mathbb{E}_{c_{\mathcal{A} \cup \{k\}}}[\mathcal{H}(Y \mid x, c_{\mathcal{A} \cup \{k\}})].
\end{align}
Finally, substituting Equation \eqref{eq:proof_expectation} back into Equation \eqref{eq:proof_mi_def}, we conclude:
\begin{equation}
    \mathcal{I}(Y; C_k \mid x, C_\mathcal{A}) \geq 0.
\end{equation}
This confirms that calculating subset distributions purely via bottom-up marginalization mathematically guarantees non-negative marginal contributions across all possible coalitions.
\end{proof}


\section{Configurations and prompts}\label{sec:Configurations}
\subsection{Dataset Description: AmbigQA}\label{sec:ambigqa_desc}

AmbigQA~\citep{ambigqa_sewon_emnlp} is an open-domain question answering dataset designed to study ambiguity in naturally occurring information-seeking questions. It is constructed from NQ-Open and targets questions that may admit multiple plausible interpretations and therefore multiple valid answers. For each ambiguous question, AmbigQA provides a set of disambiguated question--answer pairs, where each rewritten question corresponds to one specific interpretation of the original ambiguous query. For unambiguous questions, the dataset provides a single answer annotation.

This annotation structure makes AmbigQA suitable for evaluating input ambiguity, since ambiguity is defined at the level of the user question rather than at the level of model-generated outputs. In our experiments, we use the presence of multiple plausible interpretations as the ambiguity label and evaluate whether uncertainty estimates assign higher scores to ambiguous questions than to unambiguous ones.

However, AmbigQA does not provide gold annotations for ambiguous spans or their corresponding premises. Therefore, it does not support direct supervised evaluation of span-level localization or premise generation. We instead use AmbigQA for input-level ambiguity detection, while span-level uncertainty attribution is evaluated through the resulting aggregate ambiguity score.

\subsection{Dataset Description: AmbiEnt}\label{sec:ambient_desc}

AmbiEnt~\citep{ambient-liu-etal-2023-afraid} is a natural language inference (NLI) benchmark designed to study ambiguity through its effect on entailment relations. Each example consists of a premise--hypothesis pair, where the task is to determine whether the premise entails, contradicts, or is neutral with respect to the hypothesis. Unlike standard NLI datasets that assume a single gold label for each example, AmbiEnt allows an example to be associated with multiple plausible NLI labels when different interpretations of the premise and/or hypothesis lead to different entailment relations.

The dataset contains linguist-annotated examples covering diverse forms of ambiguity, including lexical, syntactic, and pragmatic ambiguity. For ambiguous examples, AmbiEnt provides annotations that characterize the possible interpretations of the input and their corresponding entailment labels. This structure makes the dataset suitable for evaluating whether a model can recognize when an input admits multiple valid readings, rather than forcing all examples into a single deterministic NLI label.

In our setting, AmbiEnt is useful because ambiguity is defined over the input pair rather than over model-generated outputs. We therefore use AmbiEnt to evaluate whether aleatoric uncertainty estimates assign higher uncertainty to premise--hypothesis pairs with multiple plausible interpretations than to those with a single interpretation.

However, AmbiEnt is not designed as a supervised benchmark for span-level uncertainty attribution. Although the dataset provides ambiguity-sensitive NLI annotations, it does not directly provide gold uncertainty scores for individual spans. Therefore, we use AmbiEnt for input-level ambiguity detection, while the span-level attributions produced by our method are assessed through their aggregate uncertainty signal.

\subsection{Configurations}\label{sec:conf}
Following the evaluation dataset sampling protocol of ICE~\citep{ICE_hou_arxiv}, we conduct both ambiguity detection and uncertainty-guided clarification experiments on AmbigQA. For AmbiEnt, we construct a balanced evaluation set of 150 examples by matching the number of unambiguous and ambiguous instances. The ambiguous set includes premise-only, hypothesis-only, and premise–hypothesis ambiguities in equal proportions. For Self-Consistency, Sample Repetition, and Sample Diversity, we follow the same implementation protocol as~\citep{ICE_hou_arxiv} and use the configurations provided in~\citep{ICE_hou_arxiv}. For Self-Consistency, we use temperature 0.7 and sample 10 answers from the model, and use 1 minus the resulting confidence score as the ambiguity score. For Sample Diversity, we use temperature 0.5 and sample 10 answers, then compute the number of unique answers as the ambiguity score. For Sample Repetition, following the official implementation, we first generate an initial answer using greedy decoding and then resample 10 answers with temperature 0.5 to estimate how frequently the initial answer is reproduced; lower repetition corresponds to higher ambiguity.

For both \proposed and ICE~\citep{ICE_hou_arxiv}, we use the same sampling configuration to ensure a fair comparison. The number of generated premises or clarifications is set to 3. For each premise or clarification condition, the \texttt{Answerer} samples 5 answers. The \texttt{Answerer} temperature is set to 0.7, while the \texttt{Clarifier} in ICE and the Premise \texttt{Generator} in \proposed use temperature 0.9. Following~\citep{ICE_hou_arxiv}, we apply LLM-based semantic clustering to group semantically equivalent outputs for ICE and Deep Ensembles, using the prompts provided in~\citep{ICE_hou_arxiv}. Following~\citep{ICE_hou_arxiv}, we also apply the same unknown-output handling procedure to \proposed, ICE, and Deep Ensembles before computing entropy, ensuring that all entropy estimates are based on a consistent output-processing protocol.

\subsection{Prompts}\label{sec:AmbigQA_Prompts}

We list the AmbigQA prompts used in our experiments as follows. The prompts follow a unified format consisting of a role instruction, task description, explicit generation or decision rules, static in-context demonstrations, and instance-specific input fields. We use the same decomposition across tasks: the \texttt{Localizer} identifies ambiguous spans, the premise \texttt{Generator} produces clarification assumptions for each span, and the \texttt{Answerer} predicts the final task output under a given set of assumptions.

For AmbigQA, we use the \texttt{Answerer} prompt from ICE~\citep{ICE_hou_arxiv} and keep it identical across all methods to ensure a fair comparison. The \texttt{Answerer} also uses the same 6-shot static demonstrations from the official ICE code for all methods. For the ICE clarification model, we use the same 3-shot static demonstrations as those used for the premise \texttt{Generator} in \proposed. For Ask4Conf-D and ICE, we adopt the corresponding prompts from ICE~\citep{ICE_hou_arxiv}.

For \proposed, AmbigQA provides gold final answers but does not include gold annotations for the intermediate outputs required by our decomposition, namely oracle spans for the \texttt{Localizer} and oracle premises for the premise \texttt{Generator}. Thus, direct comparison against dataset-provided gold labels is not applicable for these intermediate components. We therefore manually annotate a small set of instances to construct static in-context demonstrations for \proposed, following predefined annotation guidelines consistent with our task decomposition. These manually annotated instances are used solely for prompting and are excluded from all evaluation data.

For AmbiEnt, we use the same prompt interface and decomposition as AmbigQA. Since the full NLI prompts differ mainly in task-specific rule blocks and static demonstrations, we do not duplicate the lengthy prompt text here. The changes are limited to: (i) replacing answer-changing QA ambiguity rules with label-changing NLI ambiguity rules, where ambiguity is defined as an interpretation difference that can change the relation between \texttt{entailment}, \texttt{neutral}, and \texttt{contradiction}; (ii) replacing the question input field with premise--hypothesis input fields; and (iii) replacing the AmbigQA in-context demonstrations with manually constructed AmbiEnt-NLI demonstrations in the same format. These NLI demonstrations cover unambiguous, premise-ambiguous, hypothesis-ambiguous, and both-side ambiguous cases, and are also excluded from the evaluation set.

This adaptation is intended only to express the same decomposition in the NLI task format, rather than to introduce method-specific prompting advantages. Within each dataset, shared modules such as the \texttt{Answerer} are kept fixed across all compared methods, and the AmbiEnt-NLI variants preserve the AmbigQA prompting interface while changing only the rules and demonstrations needed to define task-specific ambiguity.

We list the prompts used in our experiments as follows:

\begin{itemize}
    \item The \texttt{Localizer} prompt is shown in Figure~\ref{fig:localizer_prompt}.
    \item The Premise \texttt{Generator} prompt is shown in Figure~\ref{fig:premiser_prompt}.
    \item The \texttt{Answerer} prompt is shown in Figure~\ref{fig:answerer_prompt}.
    \item The \texttt{Clusterer} prompt is shown in Figure~\ref{fig:Clusterer_prompt}.
    \item The Ask4Conf-D prompt is shown in Figure~\ref{fig:ask4conf_prompt}.
\end{itemize}

\subsection{Prompts and Configurations for Uncertainty-guided Clarification}\label{sec:Uncertainty-guided Clarification}
We used same configurations for Uncertainty-guided Clarification as subsection \ref{sec:conf}.
We list the prompts and formats used in our Uncertainty-guided Clarification experiments as follows:

\begin{itemize}
    \item Uncertainty-guided Question Clarification Prompt for Baseline Methods~\ref{fig:Uncertainty-guided Clarification_prompt}.
    \item Uncertainty-guided Question Clarification Prompt for Baseline Methods~\ref{fig:Localized Uncertainty-guided Clarification_prompt}.
    \item Uncertainty Context Format for Uncertainty-guided Question Clarification~\ref{fig:Uncertainty Context}.
    \item Uncertainty Context Format for Uncertainty-guided Question Clarification with \proposed~\ref{fig:Localized Context}.
\end{itemize}

\section{Additional Results}
\label{sec:Additional Result}

\subsection{Additional Results on Ambiguity Detection}
\label{sec:Appendix Ambiguity}

\paragraph{Uncertainty Distribution Analysis.}
Figures~\ref{fig:AmbiEnt_GPT54},~\ref{fig:AmbiEnt_31_25},~\ref{fig:GPT4fullplot}, \ref{fig:reasoning_fullplot} visualize the aleatoric uncertainty distributions for ambiguous and
unambiguous samples across different \texttt{Answerer} and \texttt{Clusterer} backbone models.
Across these backbones, \proposed exhibits a similar distributional pattern: ambiguous samples
tend to receive higher uncertainty than unambiguous samples. This indicates that the uncertainty
estimated by \proposed remains consistently aligned with input ambiguity across different model
backbones.

Although the degree of separation varies across models, the overall trend remains stable. In
particular, \proposed generally produces more distinguishable distributions between ambiguous and
unambiguous questions than the baseline methods, supporting the robustness of localized aleatoric
uncertainty estimation.


\begin{figure}[t]
\centering
\includegraphics[width=0.8\linewidth]{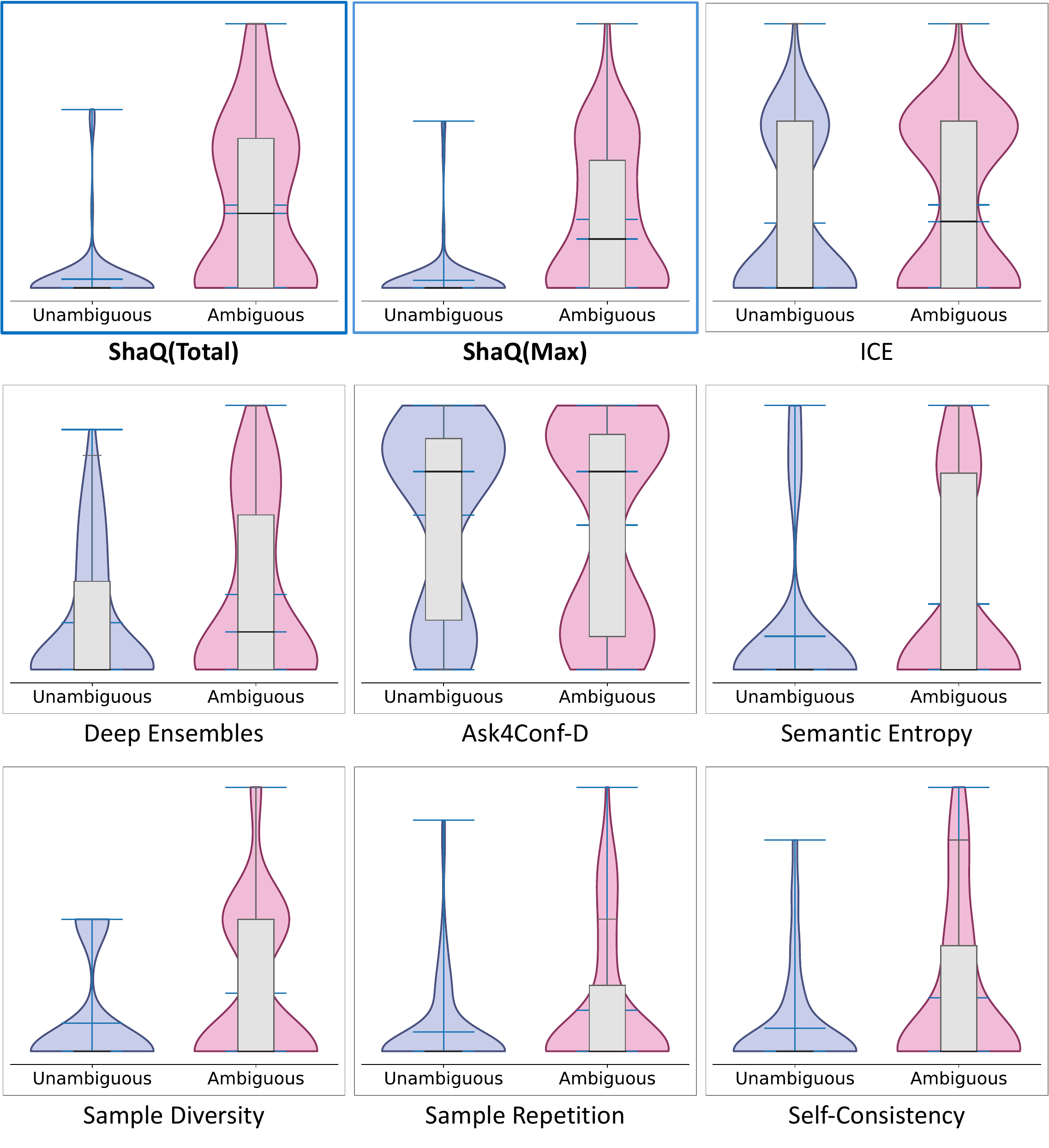}
\caption{Aleatoric uncertainty distributions on AmbiEnt (GPT-5.4-mini).}
\label{fig:AmbiEnt_GPT54}
\vspace{-1.5em}
\end{figure}
\newpage
\begin{figure}[t]
\centering
\includegraphics[width=0.8\linewidth]{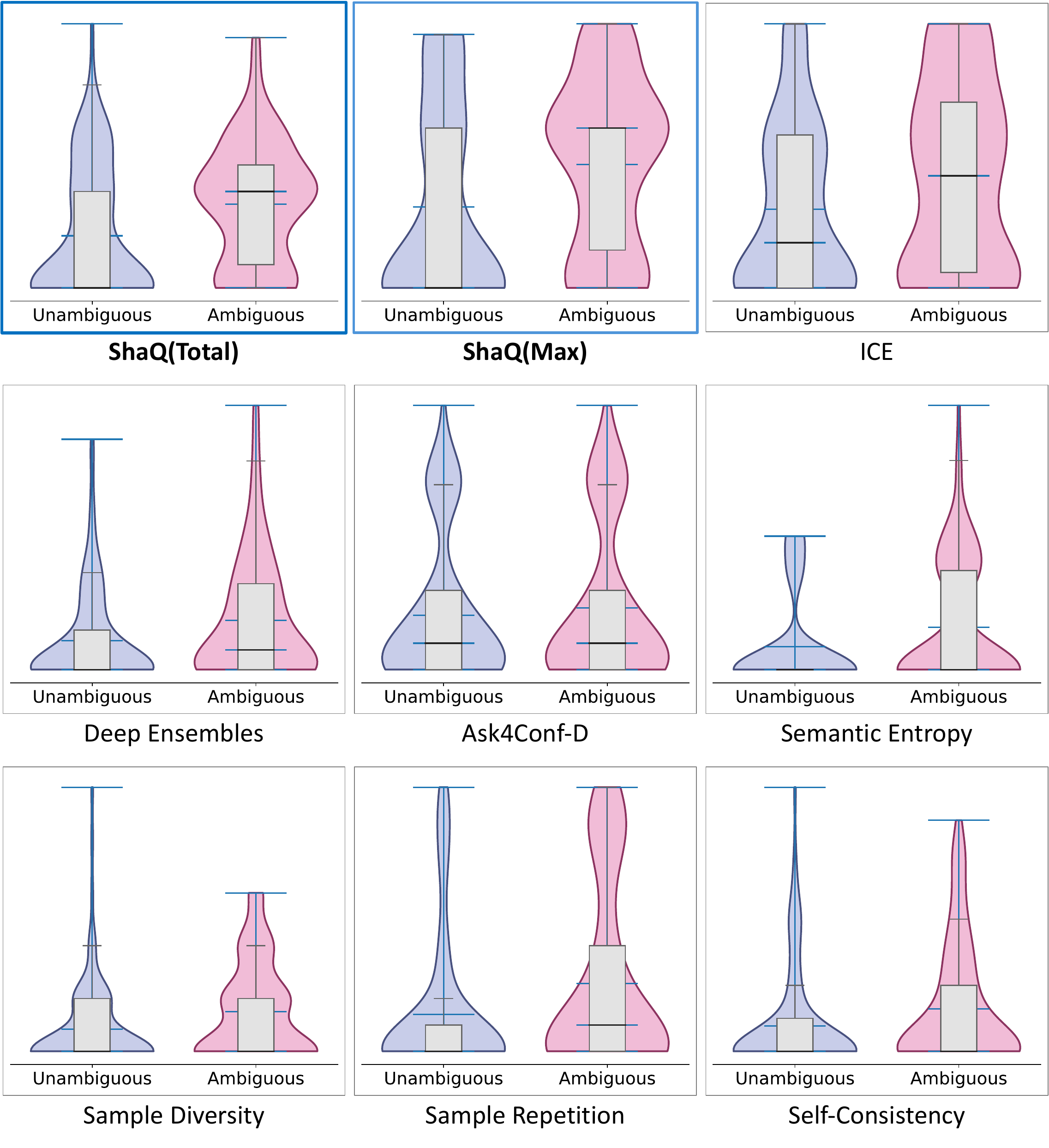}
\caption{Aleatoric uncertainty distributions on AmbigQA (GPT-4).}
\label{fig:GPT4fullplot}
\vspace{-1.5em}
\end{figure}
\clearpage
\newpage

\begin{figure}[t]
\centering
\includegraphics[width=\linewidth]{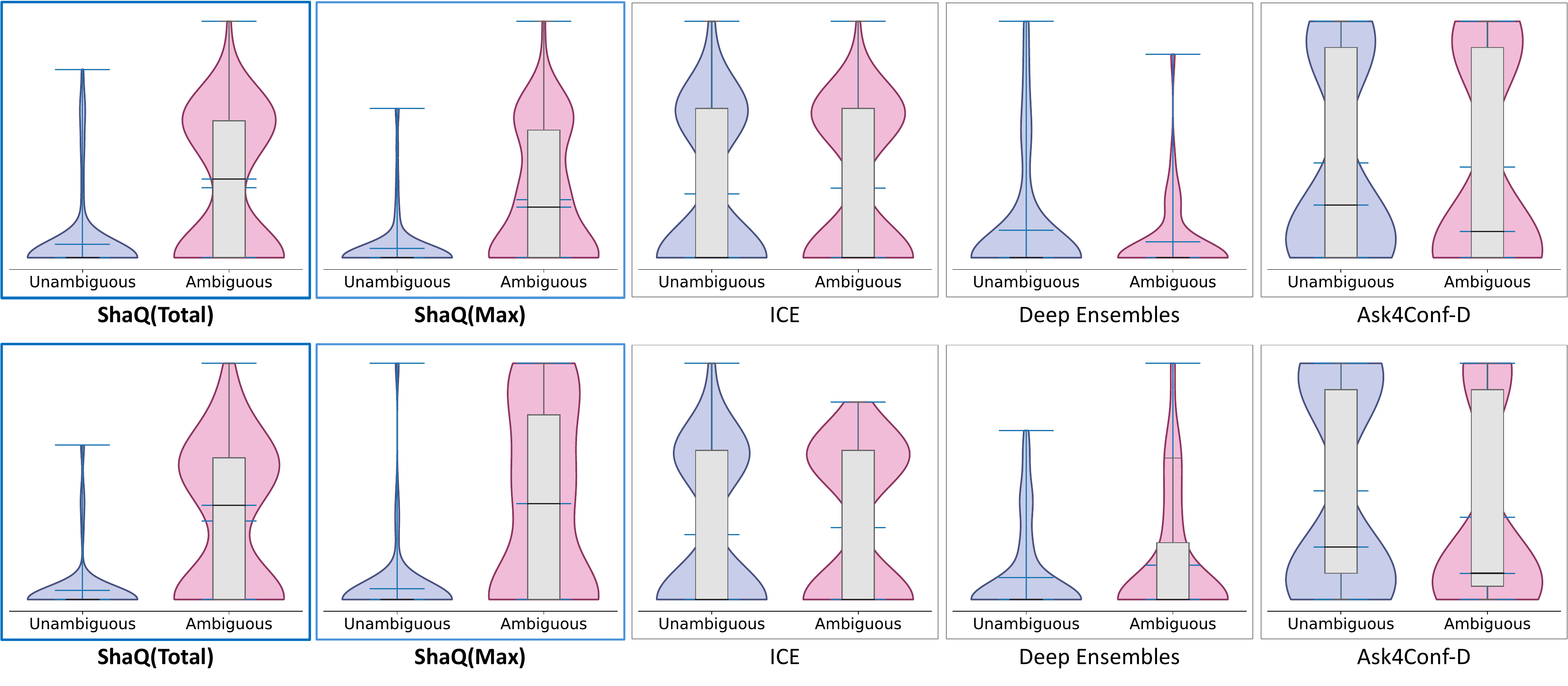}
\caption{Aleatoric uncertainty distributions on AmbiEnt. Upper row: \texttt{Gemini-3.1-flash-preview}, Lower row: \texttt{Gemini-2.5-flash.}}
\label{fig:AmbiEnt_31_25}
\vspace{-1.5em}
\end{figure}

\begin{figure}[t]
\centering
\includegraphics[width=\linewidth]{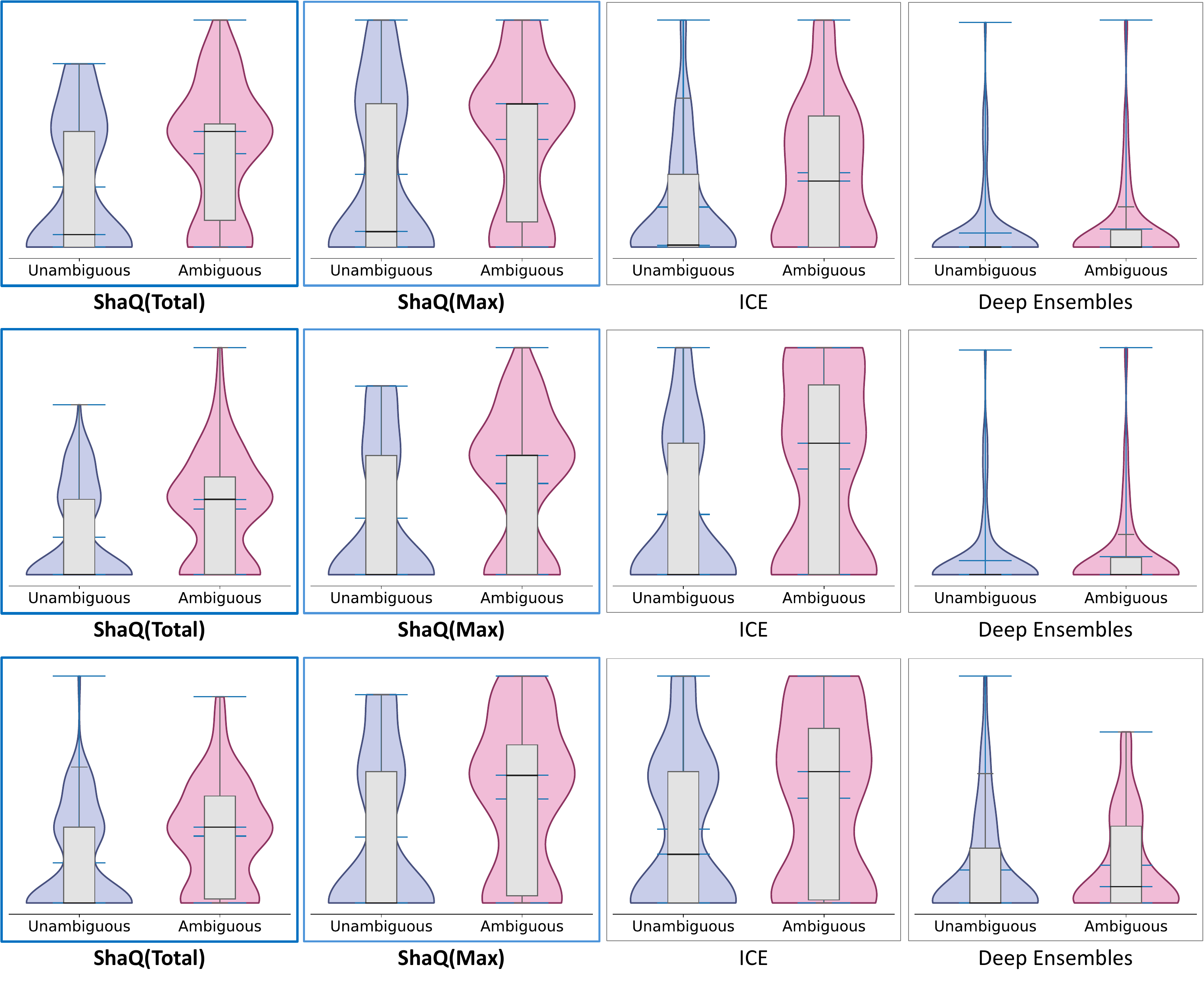}
\caption{Aleatoric uncertainty distributions on AmbigQA. First row: \texttt{GPT-5.4-mini}, Mid row: \texttt{Gemini-3.1-flash-preview}, Last row: \texttt{Gemini-2.5-flash.}}
\label{fig:reasoning_fullplot}
\vspace{-1.5em}
\end{figure}

\clearpage
\newpage

\subsection{Additional Ablation Results on Ambiguity Detection}
\label{sec:Appendix Ablation}
\begin{table}[t]
    \caption{Without semantic clustering and \texttt{Clusterer} result at AmbiEnt.}
    \centering
    \resizebox{0.5\linewidth}{!}{%
        \begin{tabular}{lccc}
\toprule
\rowcolor{gray!8}
\textit{Method} & \textbf{F1 Score} & \textbf{AUROC} & \textbf{AUPRC} \\
\midrule
\multicolumn{4}{l}{\texttt{GPT-5.4-mini}} \\
Direct SE  & 42.20 & 57.56 & 54.36 \\
ICE  & 58.66 & 56.52 & 54.04 \\
Deep Ensembles  & 56.00 & 56.71 & 55.41 \\
LOO  & 66.07 & 49.12 & 57.59 \\
LOO(Max Span)  & 66.07 & 56.92 & 63.24 \\
\midrule
\rowcolor{blue!6}
\textbf{\proposed(Ours, Max Span)}  & \textbf{74.80} & \textbf{78.18} & \underline{75.72} \\
\rowcolor{blue!8}
\textbf{\proposed(Ours)}  & \textbf{74.80} & \underline{78.06} & \textbf{76.30} \\
\midrule
\multicolumn{4}{l}{\texttt{Gemini-3.1-flash-lite-preview}} \\
LOO  & \textbf{66.07} & 53.23 & 55.63 \\
LOO(Max Span)  & 65.47 & 52.94 & 56.20 \\
\midrule
\rowcolor{blue!6}
\textbf{\proposed(Ours, Max Span)}  & \underline{65.59} & \textbf{72.39} & \textbf{70.65} \\
\rowcolor{blue!8}
\textbf{\proposed(Ours)}  & \underline{65.59} & \underline{71.96} & \underline{69.32} \\
\midrule
\multicolumn{4}{l}{\texttt{Gemini-2.5-flash}} \\
LOO  & 66.07 & 38.64 & 44.39 \\
LOO(Max Span)  & 65.47 & 41.48 & 46.22 \\
\midrule
\rowcolor{blue!6}
\textbf{\proposed(Ours, Max Span)}  & \textbf{71.42} & \underline{76.40} & \underline{74.52} \\
\rowcolor{blue!8}
\textbf{\proposed(Ours)}  & \textbf{71.42} & \textbf{76.47} & \textbf{75.78} \\
\bottomrule
\end{tabular}
    }
    \label{tab:Many_Span_Case_Appendix_NLI_ablation}
\end{table}
As an ablation, we perform ambiguity detection on AmbiEnt, an NLI task whose answer space is limited to three labels: entailment, contradiction, and neutral. Since the possible outputs are already defined by these discrete labels, we omit semantic clustering and the \texttt{Clusterer}. The results at Table~\ref{tab:Many_Span_Case_Appendix_NLI_ablation} show the same overall trend as Tables~\ref{tab:main_table} and~\ref{tab:main_table2_NLI}, suggesting that when semantic clustering is unnecessary, the computational cost of LLM-based estimation can be reduced.

\subsection{Additional Results on Multi Span Ambiguity Detection}
\label{sec:Appendix Multi Span}

\paragraph{Full Results of Multi-Span Ambiguity Detection (AmbiEnt).}
\begin{table}[t]
    \caption{Ambiguity detection results on multi-span examples ($|\mathcal{S}| \geq 2$) from AmbiEnt.}
    \centering
    \resizebox{0.5\linewidth}{!}{%
        \begin{tabular}{lccc}
\toprule
\rowcolor{gray!8}
\textit{Method} & \textbf{F1 Score} & \textbf{AUROC} & \textbf{AUPRC} \\
\midrule
\multicolumn{4}{l}{\texttt{GPT-5.4-mini}} \\
Semantic Entropy  & 39.13 & 57.07 & 78.65 \\
Sample Diversity  & 38.29 & 53.66 & 78.40 \\
Sample Repetition  & 45.83 & 60.98 & 82.47 \\
Self-Consistency  & 45.83 & 60.85 & 82.12 \\
Ask4Conf-D  & 83.95 & 40.40 & 78.13 \\
Deep Ensembles  & 58.62 & 45.95 & 75.25 \\
ICE  & 59.64 & 53.40 & 77.54 \\
LOO  & 85.36 & 49.11 & 79.59 \\
LOO(Max Span)  & 85.36 & 51.64 & 80.26 \\
\midrule
\rowcolor{blue!6}
\textbf{\proposed(Ours, Max Span)}  & \textbf{86.95} & \textbf{81.31} & \textbf{92.39} \\
\rowcolor{blue!8}
\textbf{\proposed(Ours)}  & \underline{86.56} & \underline{80.05} & \underline{91.27} \\
\midrule
\multicolumn{4}{l}{\texttt{Gemini-3.1-flash-lite-preview}} \\
Ask4Conf-D  & 58.46 & 33.45 & 73.34 \\
Deep Ensembles  & 50.00 & 40.02 & 73.21 \\
ICE  & 15.00 & 49.49 & 76.92 \\
LOO  & \textbf{85.36} & 63.38 & 84.65 \\
LOO(Max Span)  & \textbf{85.36} & 60.85 & 84.42 \\
\midrule
\rowcolor{blue!6}
\textbf{\proposed(Ours, Max Span)}  & \underline{84.05} & \textbf{74.11} & \textbf{90.80} \\
\rowcolor{blue!8}
\textbf{\proposed(Ours)}  & 82.85 & \underline{69.82} & \underline{85.27} \\
\midrule
\multicolumn{4}{l}{\texttt{Gemini-2.5-flash}} \\
Ask4Conf-D  & 78.94 & 42.29 & 76.05 \\
Deep Ensembles  & 52.63 & 40.02 & 72.00 \\
ICE  & 43.99 & 50.00 & 80.80 \\
LOO  & 85.36 & 41.16 & 73.81 \\
LOO(Max Span)  & 85.36 & 42.17 & 74.17 \\
\midrule
\rowcolor{blue!6}
\textbf{\proposed(Ours, Max Span)}  & \textbf{87.87} & \underline{85.85} & \underline{95.36} \\
\rowcolor{blue!8}
\textbf{\proposed(Ours)}  & \underline{86.95} & \textbf{86.23} & \textbf{95.47} \\
\bottomrule
\end{tabular}
    }
    \label{tab:Many_Span_Case_Appendix_NLI}
\end{table}

Table~\ref{tab:Many_Span_Case_Appendix_NLI} reports the full results on AmbiEnt examples for which the \texttt{localizer} identifies two or more candidate ambiguous spans. In this setting, \proposed and its max-span variant achieve the strongest overall performance across the evaluated backbone models, especially in AUROC and AUPRC. This indicates that the uncertainty scores produced by \proposed provide a more reliable ranking of ambiguous versus unambiguous inputs than the compared baselines.

These results highlight the importance of interaction-aware span attribution in multi-span cases. When several candidate ambiguity sources appear in the same input, the uncertainty contribution of one span may depend on how other spans are interpreted. By averaging marginal contributions over span coalitions, \proposed can capture these interactions more faithfully than methods that rely on aggregate input-level uncertainty or single-context leave-one-out perturbations. The strong performance of the max-span variant further suggests that the largest localized attribution often serves as an effective signal for detecting whether an input contains meaningful ambiguity.

\paragraph{Results of Multi-Span Ambiguity Detection (AmbigQA).}
\begin{table}[t]
\centering
\caption{Ambiguity detection results on multi-span examples ($|\mathcal{S}| \geq 2$) from AmbigQA. LOO: Leave-One-Out; DEEPENS.: Deep Ensembles; A4C: Ask4Conf-D; SE: Semantic Entropy; SDiv: Sample Diversity; SRep: Sample Repetition; SC: Self-Consistency.}
\label{tab:many_span}
\vspace{5pt}
\setlength{\tabcolsep}{2.8pt}
\renewcommand{\arraystretch}{1.08}
\footnotesize

\begin{tabular}{l|ccccccc|cccc}
\toprule
\rowcolor{gray!8}
& \multicolumn{7}{c|}{\textit{Aleatoric Methods}}
& \multicolumn{4}{c}{\textit{Other Uncertainty Methods}} \\
\rowcolor{gray!8}
\textit{Metric}
& \textbf{\proposed}
& \makecell{\textbf{\proposed}\\\textbf{Max}}
& LOO
& \makecell{LOO\\Max}
& ICE
& \makecell{DEEP\\ENS.}
& A4C-D
& SE
& SDiv
& SRep
& SC \\
\midrule
Best F1 Score
& \cellcolor{blue!8}\textbf{84.21}
& \cellcolor{blue!6}\underline{83.87}
& 82.53
& 82.53
& 75.00
& 72.00
& 68.08
& 60.46
& 76.59
& 76.00
& 75.00 \\

\bottomrule
\end{tabular}
\end{table}
\begin{table}[t]
\centering
\caption{Ambiguity detection results on multi-span examples ($|\mathcal{S}| \geq 2$) from AmbigQA with \texttt{GPT-5.4-mini}, \texttt{Gemini-2.5-flash}, and \texttt{Gemini-3.1-flash-preview}. DEEPENS.: Deep Ensembles.}
\label{tab:many_span2}
\vspace{5pt}
\setlength{\tabcolsep}{2.8pt}
\renewcommand{\arraystretch}{1.08}
\footnotesize

\begin{tabular}{l|cccc}
\toprule
\rowcolor{gray!5}
\texttt{Answerer \& Localizer Backbone}& \textbf{\proposed}& \makecell{\textbf{\proposed}\\\textbf{Max}}& ICE& \makecell{DEEP\\ENS.} \\
\midrule
\texttt{gpt-5.4-mini}& \cellcolor{blue!8}\textbf{83.87}& \cellcolor{blue!8}\underline{82.53} & 73.07 & 75.00 \\
\texttt{gemini-3.1-flash-lite-preview}& \cellcolor{blue!8}\underline{84.21}& \cellcolor{blue!8}\textbf{87.27} & 76.92 & 61.90 \\
\texttt{gemini-2.5-flash}& \cellcolor{blue!8}\underline{79.31}& \cellcolor{blue!8}\underline{79.31} & \textbf{80.70} & 68.08 \\
\bottomrule
\end{tabular}
\end{table}
Table~\ref{tab:many_span} reports the best score of the full multi-span results from AmbigQA ($|S| \geq 2$). As a result, the \proposed variants remain the strongest methods: \proposed achieves the
best F1 score, and \proposed Max obtains the second-best score. In contrast, the newly included
output-variation-based estimators, such as Sample Diversity, Sample Repetition, and
Self-Consistency, do not outperform the \proposed variants. This further supports that explicitly
localizing input-induced ambiguity provides a stronger signal than relying only on aggregate
input-level uncertainty or output variability. Table~\ref{tab:many_span2} reports the multi-span results from AmbigQA ($|S| \geq 2$) with recently released reasoning-capable models, including \texttt{GPT-5.4-mini}, \texttt{Gemini-2.5-flash}, and \texttt{Gemini-3.1-flash-preview}. Across these backbones, \proposed and its
max-span variant achieve the best average performance among the compared aleatoric methods.

\newpage
\clearpage
\begin{figure}[t]
\vspace{1.5em}
\centering
\includegraphics[width=0.8\linewidth]{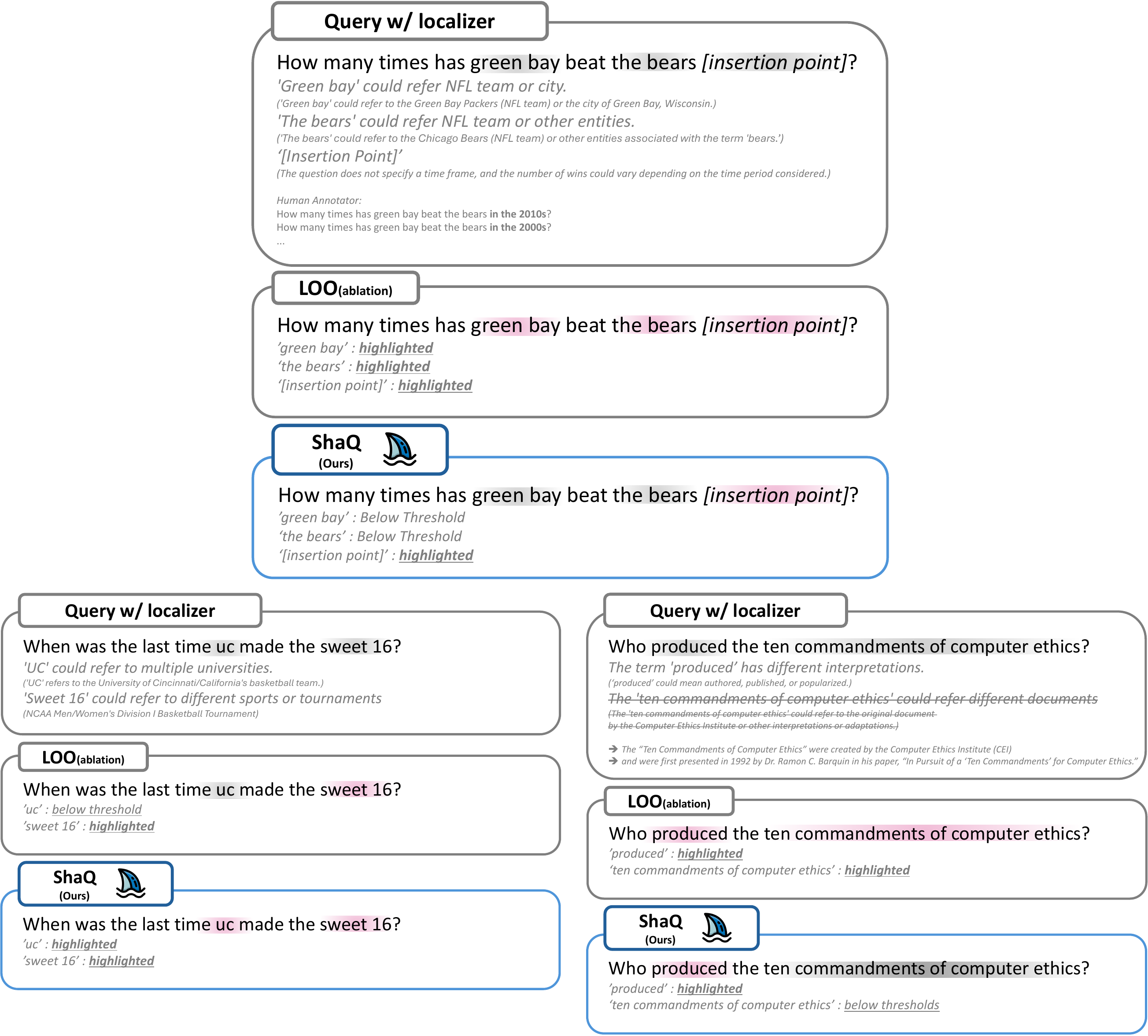}
\caption{Qualitative Result of multi span ambiguous questions on AmbigQA (GPT-4). Highlighted span is colored with pink.}
\label{fig:Case Report}
\vspace{-1.5em}
\end{figure}
\newpage
\begin{figure}[t]
\vspace{1.5em}
\centering
\includegraphics[width=\linewidth]{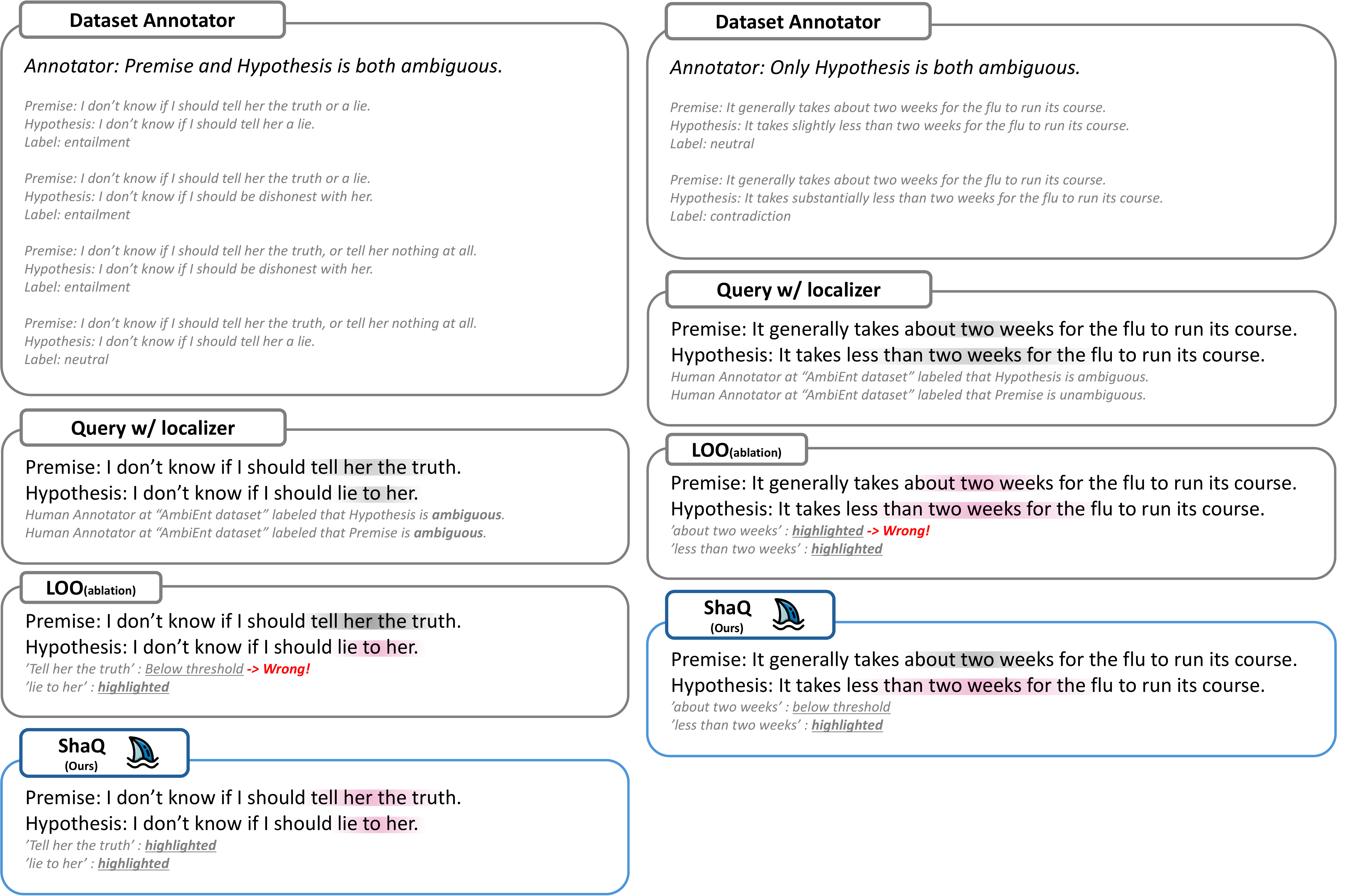}
\caption{Qualitative Result of multi span ambiguous questions on AmbiEnt (GPT-5.4-mini). Highlighted span is colored with pink.}
\label{fig:Case Report AmbiEnt}
\vspace{-1.5em}
\end{figure}

\paragraph{Qualitative Analysis on AmbigQA}

We further conduct a qualitative evaluation of \proposed on three representative multi-span cases identified by the AmbigQA localizer: 
\begin{itemize}[leftmargin=1.5em, itemsep=0pt, topsep=1pt]
\item (1) cases in which all detected spans must be revised to resolve the ambiguity 
(Figure~\ref{fig:Case Report}, bottom left)
\item (2) cases in which revising a single span resolves the ambiguity associated with the remaining span(s) 
(Figure~\ref{fig:Case Report}, top)
\item (3) cases in which the \texttt{localizer} incorrectly identifies an unambiguous span as ambiguous 
(Figure~\ref{fig:Case Report}, bottom right).
\end{itemize}

For case (1), \proposed successfully highlights both spans requiring clarification, while LOO fails to identify them correctly. 
For case (2), \proposed highlights the appropriate insertion point where adding information resolves the ambiguity of the other span; this behavior is consistent with the AmbigQA human annotation, which also resolves the ambiguity by introducing additional information. In contrast, LOO highlights all detected spans. 
For case (3), despite the \texttt{localizer} falsely marking a clear span as ambiguous, \proposed avoids highlighting the unambiguous proper noun ``ten commandments of computer ethics'' and instead highlights only the genuinely ambiguous part, whereas LOO highlights the clear span as well. 
This indicates that, although the initial localization step relies on an LLM and may over-identify candidate spans, \proposed can assign near-zero uncertainty contribution to spans whose clarification does not reduce predictive uncertainty.

These examples show that \proposed can help filter out unnecessary or incorrectly localized spans by estimating their actual contribution to predictive uncertainty, even when the \texttt{localizer} over-identifies candidate spans or marks clear expressions as ambiguous. Overall, this suggests that \proposed provides more precise and reliable fine-grained aleatoric uncertainty quantification than LOO.



\paragraph{Qualitative Analysis on AmbiEnt}
In AmbiEnt, ambiguity may arise independently from either the premise or the hypothesis. 
We therefore conduct a qualitative evaluation of \proposed on two representative multi-span cases identified by the AmbiEnt localizer:
\begin{itemize}[leftmargin=1.5em, itemsep=0pt, topsep=1pt]
\item (1) cases in which all detected spans must be revised to resolve the ambiguity 
(Figure~\ref{fig:Case Report AmbiEnt}, left);
\item (2) cases in which the \texttt{localizer} incorrectly identifies an unambiguous span as ambiguous 
(Figure~\ref{fig:Case Report AmbiEnt}, right).
\end{itemize}

Although neither AmbiEnt nor AmbigQA provides span-level ambiguity annotations, AmbiEnt includes annotations indicating whether the premise and hypothesis are ambiguous, respectively. 
We use these annotations as a weak reference for evaluating span localization: if the annotator highlights a region in the premise or hypothesis, we regard the corresponding sentence component as ambiguous; otherwise, we treat it as unambiguous. 
This allows us to compare whether the spans highlighted by each method are consistent with the ambiguity labels assigned at the premise and hypothesis level.

For case (1), \proposed successfully highlights both spans that require clarification, whereas LOO fails to identify them correctly. 
For case (2), despite the \texttt{localizer} falsely marking the unambiguous span ``about two weeks'' as a candidate ambiguous span, \proposed avoids assigning a high uncertainty score to this clear span and instead highlights only the genuinely ambiguous part. 
In contrast, LOO also highlights ``about two weeks'', suggesting that it is more vulnerable to false-positive candidates produced by the localization stage. 
These results indicate that, although the initial localization step relies on an LLM and may over-identify candidate spans, \proposed can assign near-zero uncertainty contribution to spans whose clarification does not reduce predictive uncertainty. 
This behavior demonstrates the robustness of \proposed as a post-localization uncertainty attribution method, particularly in multi-span settings where ambiguous and unambiguous candidate spans may coexist.

\paragraph{Shapley-style interpretation of multi-span cases.}
The qualitative examples can be interpreted through the Shapley-style marginal contribution view underlying \proposed. 
In our formulation, each localized span is treated as a player for attribution purposes, and its score is determined not by whether the span appears ambiguous in isolation, but by its average marginal reduction of predictive entropy across clarification contexts. 
This perspective helps explain several recurring multi-span patterns.

First, when the \texttt{localizer} includes an unambiguous span together with a genuinely ambiguous span, the unambiguous span behaves as a near-dummy player: clarifying it produces little or no reduction in predictive entropy under most contexts. 
This pattern appears in AmbigQA Case (3) and AmbiEnt Case (2), where the localizer over-identifies a clear span as a candidate ambiguity. 
For example, in AmbiEnt Case (2), the span ``about two weeks'' is initially detected as a candidate ambiguous span, but clarifying it does not meaningfully reduce predictive uncertainty. 
\proposed therefore assigns near-zero attribution to such spans, even if they were initially marked as candidates by the localizer.

Second, when multiple spans are complementary sources of ambiguity and the input becomes clear only after they are jointly clarified, the uncertainty reduction is shared across the corresponding spans. 
This pattern corresponds to AmbigQA Case (1) and AmbiEnt Case (1), where all detected spans must be revised to fully resolve the ambiguity. 
In this case, no single span fully accounts for the ambiguity in isolation, and the Shapley-style averaging distributes the joint uncertainty according to each span's marginal contribution across different clarification contexts.

Third, when multiple spans are partially redundant--that is, clarifying one span can substantially reduce the ambiguity associated with the others--the attribution need not be uniformly split. 
This pattern is illustrated by AmbigQA Case (2), where revising a single span or insertion point can resolve ambiguity associated with the remaining span(s). 
If the spans are symmetric in their effects on the answer distribution, the attribution is expected to be shared. 
However, if one span induces a richer set of plausible interpretations or yields larger entropy reductions across more contexts, \proposed assigns a larger attribution to that span. 
Thus, concentrated attribution reflects an asymmetric contribution to the model's predictive uncertainty, rather than an artifact of the localizer.

\newpage
\clearpage
\subsection{Additional Results on Uncertainty-guided Clarification}  \label{sec:Appendix Clarification}
\begin{table}[t]
\centering

\begin{minipage}{0.48\textwidth}
\centering
\caption{High \proposed Max-Span Aleatoric Uncertainty Set ($> 1$).}
\label{tab:appendix_max_span}
\resizebox{\linewidth}{!}{%

\begin{tabular}{lccc}
\toprule
\rowcolor{gray!8}
\textit{Method} & Clarified & $\Delta H$($\uparrow$) & Edit ($\downarrow$)\\
\midrule
Semantic Entropy  &  0.3429 & 0.1115 & 5.8571 \\
ICE  &  0.5397 & -0.0852 & 7.2857\\
Deep Ensembles  &  0.5245 & -0.0700 & 7.3571\\
\midrule
\rowcolor{blue!8}
\textbf{\proposed(Ours)}&   \textbf{0.3296} & \textbf{0.1248} & \textbf{4.9285}\\
\bottomrule
\end{tabular}

}
\end{minipage}
\hfill
\begin{minipage}{0.48\textwidth}
\centering
\caption{High \proposed Max-Span Aleatoric Uncertainty Set (top 20\%).}
\label{tab:max_span}
\resizebox{\linewidth}{!}{%

\begin{tabular}{lccc}
\toprule
\rowcolor{gray!8}
\textit{Method} & Clarified & $\Delta H$($\uparrow$) & Edit ($\downarrow$)\\
\midrule
Semantic Entropy  &  0.3165 & 0.0654 & 6.7142 \\
ICE  &  0.4872 & -0.1051 & 7.1428\\
Deep Ensembles &  0.4426 & -0.0605 & 7.1428\\
\midrule
\rowcolor{blue!8}
\textbf{\proposed(Ours)}&   \textbf{0.3159} & \textbf{0.0661} & \textbf{5.0476}\\
\bottomrule
\end{tabular}

}
\end{minipage}

\vspace{6pt}

\begin{minipage}{0.48\textwidth}
\centering
\caption{High Baseline Max-Total Aleatoric Uncertainty Set (top 30).}
\label{tab:appendix_High_Baseline}
\resizebox{\linewidth}{!}{%

\begin{tabular}{lccc}
\toprule
\rowcolor{gray!8}
\textit{Method} & Clarified & $\Delta H$($\uparrow$) & Edit ($\downarrow$)\\
\midrule
Semantic Entropy  &  0.3374 & -0.0174 & 6.1333 \\
ICE  &  0.3477 & -0.0277 & 6.6333\\
Deep Ensembles  &  0.3050 & 0.0149 & 7.5000\\
\midrule
\rowcolor{blue!8}
\textbf{\proposed(Ours)}&   \textbf{0.2848} & \textbf{0.0351} & \textbf{5.5333}\\
\bottomrule
\end{tabular}

}
\end{minipage}
\hfill
\begin{minipage}{0.48\textwidth}
\centering
\caption{High Baseline Max-Total Aleatoric Uncertainty Set (top 20\%).}
\label{tab:baseline_aleatoric}
\resizebox{\linewidth}{!}{%

\begin{tabular}{lccc}
\toprule
\rowcolor{gray!8}
\textit{Method} & Clarified & $\Delta H$($\uparrow$) & Edit ($\downarrow$)\\
\midrule
Semantic Entropy  &  0.3489 & 0.0066 & 6.1428 \\
ICE  &  0.3477 & 0.0207 & 6.7142\\
Deep Ensembles  &  0.3239 & 0.0315 & 7.8095\\
\midrule
\rowcolor{blue!8}
\textbf{\proposed(Ours)}&  \textbf{0.3008} & \textbf{0.0547} & \textbf{5.8571}\\
\bottomrule
\end{tabular}

}
\end{minipage}

\vspace{6pt}

\begin{minipage}{0.48\textwidth}
\centering
\caption{High Baseline Uncertainty Disagreement Set.}
\label{tab:baseline_uncertainty}
\resizebox{\linewidth}{!}{%
\begin{tabular}{lccc}
\toprule
\rowcolor{gray!8}
\textit{Method} & Clarified & $\Delta H$($\uparrow$) & Edit ($\downarrow$)\\
\midrule
Semantic Entropy  & 0.1529 & 0.0199 & 5.5714 \\
ICE  & 0.1676 & 0.0052 & 6.0952\\
Deep Ensembles  & 0.1758 & -0.0029 & 6.7619\\
\midrule
\rowcolor{blue!8}
\textbf{\proposed(Ours)}& \textbf{0.1355} & \textbf{0.0373} & \textbf{5.0476}\\
\bottomrule
\end{tabular}

}
\end{minipage}

\vspace{-8pt}
\end{table}

\paragraph{Breakdown by Selection Criterion.}
We provide a breakdown of uncertainty-guided clarification results under three complementary
high-uncertainty settings. The first two settings correspond to the two selection criteria used to
construct the evaluation superset in the Section~\ref{subsec:uncertainty_guidance}: examples selected by high \proposed max-span
aleatoric uncertainty and examples selected by high baseline max-total aleatoric uncertainty. The
third setting considers examples where baseline uncertainty estimators disagree with one another.
Together, these settings allow us to analyze whether \proposed remains effective across cases
identified by its own localized uncertainty signal, cases identified by conventional uncertainty
estimators, and cases where existing estimators provide unstable guidance.

\paragraph{High \proposed Max-Span Aleatoric Uncertainty Set.}
This setting focuses on gold ambiguous examples for which \proposed assigns high maximum
span-level aleatoric uncertainty. We consider two variants: a threshold-based subset where the
maximum span uncertainty is greater than 1, reported in Table~\ref{tab:appendix_max_span}, and
the top-20\% subset used in the main-text superset construction, reported in
Table~\ref{tab:max_span}. These examples represent cases where \proposed identifies a particularly
strong localized source of ambiguity within the input.

As shown in Tables~\ref{tab:appendix_max_span} and~\ref{tab:max_span}, \proposed-guided
clarification achieves the largest entropy reduction while requiring the fewest edits. This indicates
that the examples strongly flagged by \proposed are not merely high-scoring artifacts, but genuinely
ambiguous cases for which localized clarification is especially useful. In this regime,
baseline-guided rewrites often require larger modifications and may even fail to reduce entropy,
suggesting that these methods do not reliably identify the specific part of the input that should be
clarified. Thus, high max-span uncertainty under \proposed captures fine-grained ambiguous samples
that are difficult for aggregate uncertainty methods to handle.

\paragraph{High Baseline Max-Total Aleatoric Uncertainty Set.}
This setting focuses on gold ambiguous examples that existing uncertainty estimators strongly flag
as uncertain. We consider both a top-30 subset, reported in Table~\ref{tab:appendix_High_Baseline},
and the top-20\% subset used in the main-text superset construction, reported in
Table~\ref{tab:baseline_aleatoric}. These examples represent cases that conventional input-level or
output-level uncertainty estimators regard as difficult.

As shown in Tables~\ref{tab:appendix_High_Baseline} and~\ref{tab:baseline_aleatoric}, \proposed
again achieves the strongest clarification performance, reducing entropy more effectively while
requiring fewer edits. This shows that \proposed is not only effective on samples selected by its own
uncertainty score. Even when the evaluation examples are chosen according to baseline uncertainty
signals, \proposed can identify the localized source of ambiguity and provide more useful information
for clarification. In other words, for examples that baselines can detect as uncertain, \proposed can
still provide better guidance on how the input should be revised.

\paragraph{High Baseline Disagreement Set.}
We further evaluate on a challenging subset of gold ambiguous examples where baseline uncertainty
methods disagree, selected by the top 20\% standard deviation of baseline uncertainty scores. This
subset captures cases in which conventional input-level or output-level uncertainty estimators provide
unstable or conflicting signals, making it difficult to determine which uncertainty estimate should
guide clarification.

As shown in Table~\ref{tab:baseline_uncertainty}, \proposed achieves the largest entropy reduction
while requiring the smallest edit distance. This indicates that, even when existing uncertainty
estimators disagree, \proposed can still provide a more actionable ambiguity signal for clarification.
In particular, localized span-level uncertainty helps identify the specific ambiguity that should be
resolved, rather than relying on a single global uncertainty score that may be sensitive to the choice
of estimator.

\subsection{Additional Results on Human-to-Human Medical Multi-turn Dialogue}
\label{sec:Appendix Meditod}
\paragraph{Dataset.}
We evaluate on MediTOD~\citep{meditod_saley_emnlp}, a publicly available English dataset of staged doctor-patient interviews covering respiratory and musculoskeletal specialties, comprising 213 dialogues with 22,503 annotated utterances. Each dialogue follows the Objective Structured Clinical Examination (OSCE) format, in which medical professionals enact both doctor and patient roles to systematically elicit a patient history. Dialogues are highly conversational, averaging 96 utterances per dialogue, and are annotated with comprehensive intent and slot-attribute labels under the Comprehensive Medical Attribute Schema (CMAS).

\paragraph{Experimental Setup.}
A key distinction of this evaluation is that \proposed is applied \emph{online}: at each utterance, only the preceding dialogue history is provided as context, and future utterances are deliberately withheld from the prompt. This mirrors the real-world constraint that an interlocutor must assess potential misunderstanding in real time, without foreknowledge of how ambiguity will be resolved. Concretely, \proposed localizes ambiguous spans within each utterance conditioned on the accumulated dialogue history, and the detected spans are logged at every turn. We use GPT-4.1 as the base model throughout. Notably, the vast majority of utterances in this setting are declarative statements rather than queries, which highlights a key strength of \proposed: because its core mechanism operates by simulating clarification of an ambiguous span, it is agnostic to the surface form of the input and applies equally to questions, statements, and other utterance types. To accommodate this, we adopt a generalized version of \proposed in which the Answerer, rather than generating an answer to a query, rewrites the current declarative utterance by embedding the meaning of a given premise directly into the statement -- leaving all other components unchanged. Detailed prompts are provided in prompts~\ref{fig:medical_localizer}~\ref{fig:medical_premise_generator}~\ref{fig:medical_answerer}~\ref{fig:medical_cluster}.

\paragraph{Practical Significance.}
This setting exposes a practically important capability of \proposed: rather than performing a one-shot analysis of a complete document, it can serve as a \emph{turn-level ambiguity monitor} that alerts interlocutors to potentially problematic spans as a conversation unfolds. In the medical domain, where a misunderstood symptom descriptor or temporal qualifier can have serious downstream consequences for diagnosis and treatment, the ability to flag ambiguous spans at each turn -- before the conversation proceeds -- represents a concrete and high-value application beyond the LLM input setting for which \proposed was originally designed.

\newpage
\clearpage
\begin{figure}[t]
    \centering
    \includegraphics[width=0.9\linewidth]{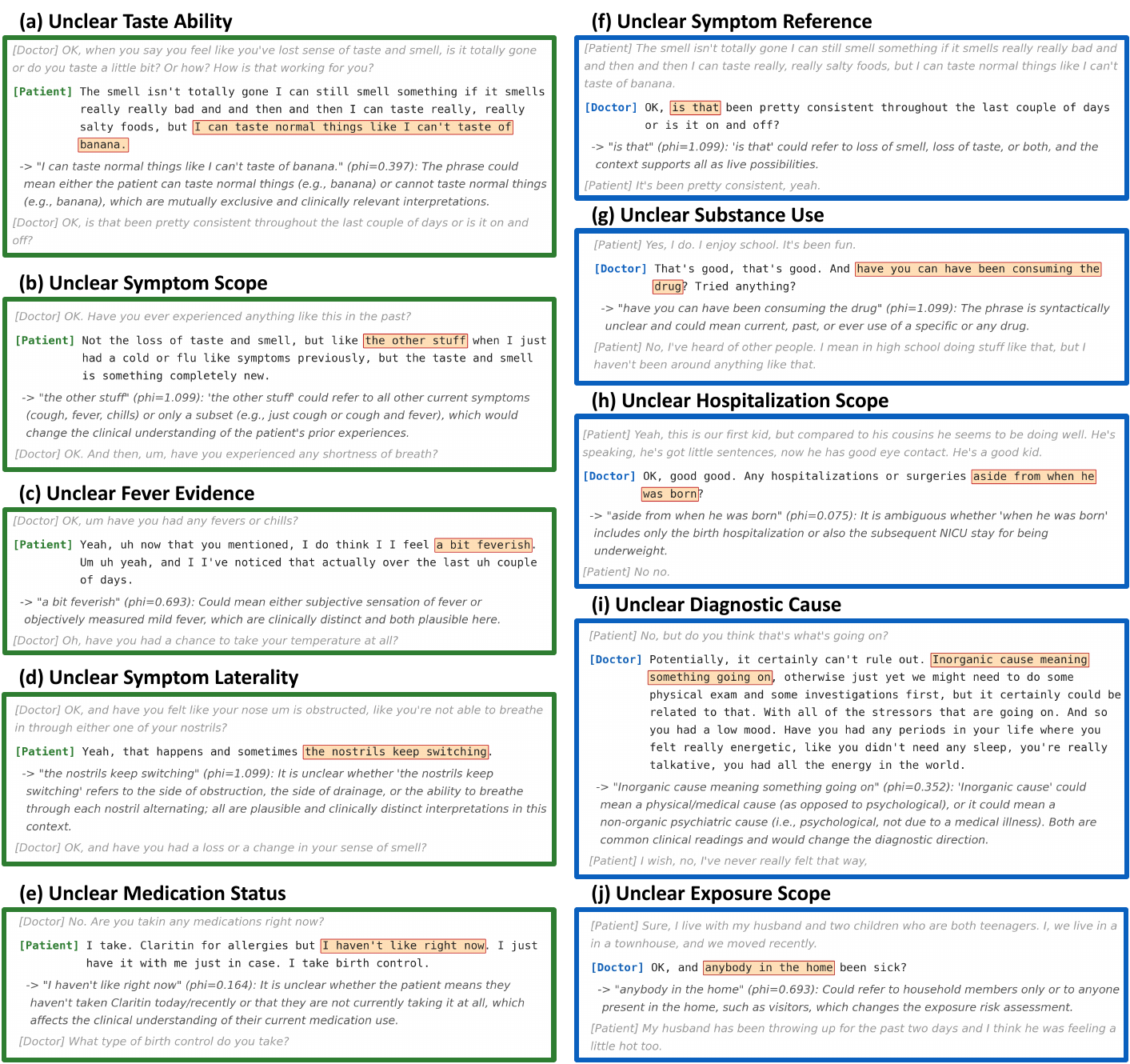}
    \caption{\textbf{Extensive Analysis For MediTOD Dialogue Set.}}
    \label{fig:Appendix_Medical}
\end{figure}
\paragraph{Extensive Medical Dialogue Case Study.}
Figure~\ref{fig:Appendix_Medical} presents ten representative cases spanning qualitatively distinct ambiguity types, demonstrating the breadth of clinically significant uncertainty that \proposed is able to localize. We organize our observations as follows.

\textbf{Diversity of ambiguity types.}
The detected spans exhibit a wide variety of linguistic and clinical ambiguity. \textit{Lexical ambiguity} arises when a single expression admits multiple medically distinct readings: in~(c), \textit{``a bit feverish''} is ambiguous between a subjective thermal sensation and a measurable low-grade fever, and in~(e), \textit{``I haven't like right now''} is ambiguous between a temporary lapse in medication adherence and complete discontinuation of Claritin. \textit{Referential ambiguity} occurs when the antecedent of a referring expression is underspecified: in~(f), the demonstrative \textit{``is that''} could resolve to loss of smell, loss of taste, or both, with the surrounding context supporting all three readings simultaneously. \textit{Scope ambiguity} manifests in~(b), where \textit{``the other stuff''} is unclear as to whether it encompasses all co-occurring symptoms or only a subset, materially affecting the clinical interpretation of the patient's prior illness history, and in~(j), where \textit{``anybody in the home''} is ambiguous between household members and any recent visitor, altering the exposure risk assessment. \textit{Syntactic ambiguity} appears in~(g), where \textit{``have you can have been consuming the drug''} is structurally malformed in a way that leaves the temporal scope--current, past, or lifetime use--and the referent--a specific drug or any drug--simultaneously underspecified.

\textbf{Ambiguity in both speakers.}
Notably, \proposed detects ambiguous spans in the utterances of \textit{both} the patient and the doctor, underscoring that miscommunication risk is not unidirectional. Patient-side cases include~(a), (b), (c), (d), and~(e), where vague or hedged self-reports carry clinically consequential ambiguity. Doctor-side cases include~(f), (g), (h), (i), and~(j), where imprecise clinical phrasing risks eliciting misaligned patient responses. For instance, in~(h), the physician's exclusion clause \textit{``aside from when he was born''} is ambiguous as to whether it covers only the birth hospitalization or also a subsequent NICU admission for being underweight, a distinction that could alter the reported medical history. In~(i), the term \textit{``Inorganic cause meaning something going on''} conflates a physical/medical etiology with a non-organic psychiatric one, representing a framing that could bias the patient's subsequent response.

\textbf{Summary.}
Across all ten cases, the identified spans are clinically meaningful: misinterpretation of any one of them could lead to an incorrect diagnosis, an inappropriate treatment recommendation, or a missed follow-up question. These results collectively demonstrate that \proposed generalizes well beyond the LLM question-answering setting for which it was designed, offering a principled and interpretable mechanism for turn-level ambiguity monitoring in real-world clinical dialogue.

\section{Computational Analysis}\label{sec:appendix_compute}
The primary computational bottleneck of \proposed lies in the \texttt{Answerer} calls required for all joint clarifications. While concatenating multiple joint clarifications into a single LLM call may seem a natural shortcut, it incurs additional cross-attention computation among the concatenated premise sets, whose cost grows quadratically with the combined sequence length. Instead, since all joint clarification combinations share the same query $x$ as a common prefix, each call effectively reduces to processing the premise suffix alone -- a property that most modern LLM inference engines exploit via \textit{prefix KV-caching}~\citep{prefixKVcache_kwon_acm}, caching the key -- value representations of $x$ and reusing them across calls. Consequently, \proposed computes Shapley values by varying only the premise clarifications in the suffix across calls, leaving the shared prefix untouched. Furthermore, beyond inference optimizations, the overall computational cost of \proposed can be significantly reduced by employing lightweight LLMs. As demonstrated in Tables~\ref{tab:main_table2} and \ref{tab:main_table2_NLI}, we evaluated \proposed using smaller, cost-efficient models such as \texttt{GPT-5.4-mini}, \texttt{Gemini-3.1-flash-lite-preview}, and \texttt{Gemini-2.5-flash} for the \texttt{Answerer} and \texttt{Clusterer} backbones. Despite the reduced model capacity, \proposed consistently outperforms baselines, demonstrating that our framework remains highly effective and computationally practical even in resource-constrained settings.

\paragraph{Compute resources and API budget.}
All LLM inference in our experiments was conducted through API-based inference via OpenRouter.
We did not run any backbone models locally. Local compute was used only for experiment orchestration,
prompt construction, response parsing, aggregation, and metric computation. The experiments were
orchestrated from a CPU machine with 64 logical CPU cores and approximately 1 TiB RAM; no
local GPU was used. Unless otherwise stated, runs used \texttt{max\_workers=5} for
example-level parallel API orchestration. For repeated-sampling calls, multiple requested
samples were issued as repeated single-completion API calls, with concurrency subject to the
same orchestration setting and OpenRouter rate limits.

We report compute in terms of OpenRouter chat-completion requests, which is the most relevant
reproducibility unit in our API-based setting. The AmbigQA main GPT-4 experiments used
approximately 20.2k API requests and took about 25 minutes of wall-clock time, excluding one
optional fresh LOO rerun; including the fresh LOO rerun adds approximately 7.9k requests and
13 minutes. The AmbiEnt main GPT-5.4-mini experiments used approximately 17.8k requests and
took about 1.0 hour. The Gemini backbone ablations used approximately 84.3k requests in total:
47.4k requests for AmbiEnt and 36.9k requests for AmbigQA, with about 2.2 hours of aggregate
wall-clock time. The clarification simulation used approximately 9.4k requests and took about
15 minutes. MediTOD was used only as a qualitative case study over a small subset of dialogue
examples, and its API budget was not separately logged. Exact provider-side token accounting was
not persisted for all runs; where token counts are needed, they can be reconstructed approximately
from the saved prompts and raw completions.

\section{Limitations and Future Directions}\label{sec:appendix_limitations}
While \proposed provides a principled and mathematically rigorous framework for localizing input-induced uncertainty, it opens several promising avenues for future exploration. In our current formulation, we generate premises for each ambiguous span independently and employ a bottom-up marginalization strategy to compute Shapley values. This design choice aligns with the \textit{Marginal} (or interventional) Shapley value approach widely used in feature attribution. By doing so, we ensure mathematical tractability, guarantee non-negative marginal information gains, and maintain computational efficiency without the overhead of complex conditional sampling.

However, in highly intricate natural language contexts, the plausible interpretations of one span can be tightly coupled with those of another. A compelling future direction is to explicitly model these semantic correlations by introducing a \textit{Premise Generation Checker} to dynamically filter logically inconsistent premise combinations during the sampling phase. This naturally extends our framework into the realm of \textit{Conditional} Shapley estimation. Just as the broader machine learning community actively debates the trade-offs between Marginal and Conditional Shapley values—balancing the causal isolation of individual features against strict adherence to the data manifold—this discourse can be elegantly extended to ambiguity localization in LLMs. Investigating how such conditional premise sampling impacts both the attribution fidelity and the computational cost of span-level uncertainty remains an exciting frontier for future work.

\section{Broader Impacts}
\label{sec:broader_impacts}
This work addresses a fundamental limitation of existing uncertainty quantification methods for LLMs: the inability to identify which part of an input drives predictive uncertainty. By providing span-level uncertainty attribution, \proposed enables more transparent and actionable human-AI interaction, which is especially valuable in real-world deployments where improving reliability must rely on resolving input ambiguity rather than modifying model parameters. The positive societal impacts of \proposed are most pronounced in high-stakes domains. As demonstrated in our MediTOD evaluation, \proposed can serve as a turn-level ambiguity monitor in clinical dialogues, flagging underspecified symptom descriptors or temporal qualifiers before a conversation proceeds. In such settings, targeted localization of ambiguous spans can help prevent misdiagnoses or inappropriate treatment recommendations arising from miscommunication between patients and clinicians. More broadly, \proposed provides a principled mechanism for LLMs to communicate where users should focus their clarification efforts, reducing unnecessary user effort while improving the reliability of downstream decisions. Beyond the medical domain, \proposed has potential applications wherever LLMs are deployed as decision-support tools and input quality directly affects output reliability -- including legal document review, customer service, and educational tutoring systems. In these contexts, fine-grained uncertainty localization can enhance human oversight by drawing attention to the specific portions of an input that most warrant revision or verification. As with any method that automatically identifies ambiguous content in user inputs, \proposed should be deployed with appropriate human oversight, particularly in sensitive domains where incorrect localization could mislead users. We recommend that practitioners treat \proposed's span-level attributions as decision-support signals rather than definitive judgments, and that the system be evaluated on domain-specific data before deployment in high-stakes settings.

\section{Asset Licenses and Terms of Use}
\label{Asset Licenses and Terms of Use}

We use only existing datasets, baselines, and proprietary model APIs for research evaluation.
All datasets, baselines, and prior methods used in this work are credited through citations in the
main paper and references. We use AmbigQA~\citep{ambigqa_sewon_emnlp},
AmbiEnt~\citep{ambient-liu-etal-2023-afraid}, and
MediTOD~\citep{meditod_saley_emnlp} only for non-commercial academic research and evaluation,
following their intended research use. We do not redistribute the original datasets; instead, we
provide instructions and scripts for obtaining them from their official sources.

AmbigQA/AmbigNQ is based on NQ-Open and is distributed for research use; the HuggingFace
mirror lists the dataset under CC BY-SA 3.0. AmbiEnt is released under CC BY 4.0 according to
its official repository. MediTOD is an existing medical dialogue benchmark constructed from
simulated doctor-patient dialogues; its source dialogues are described as simulated and publicly
available under CC0 in the original dataset paper. We cite the original creators of all datasets and
respect the corresponding license and terms-of-use requirements.

For baselines and prior methods, including ICE, Deep Ensembles, Ask4Conf-D, Semantic Entropy,
Sample Diversity, Sample Repetition, and Self-Consistency, we cite the original papers and
implement or adapt them only for comparative research evaluation.

For proprietary LLMs, we accessed GPT-based and Gemini-based models through OpenRouter, a
third-party API gateway for LLM providers, rather than directly through individual provider APIs.
Specifically, our experiments used GPT-4, GPT-5.4-mini, Gemini-2.5-flash, and
Gemini-3.1-flash-lite-preview. These models were used only for inference in our experimental
pipeline, including span localization, premise generation, answer generation, semantic clustering,
and baseline evaluation. Our use of these models is subject to the OpenRouter Terms of Service
and, where applicable, the terms of service, usage policies, and data-handling policies of the
underlying model providers.

We do not redistribute proprietary model weights, modify or fine-tune the underlying proprietary
models, or claim ownership of the models themselves. Generated outputs are used solely for
research evaluation and analysis. We will comply with any additional attribution, safety, usage,
and publication requirements specified by OpenRouter and the corresponding model providers.

\newpage
\begin{figure}[t]
    \centering
    \input{Prompts/Localizer_Prompt}
    \caption{\textbf{Prompt used for the Localizer.}}
    \label{fig:localizer_prompt}
\end{figure}

\begin{figure}[t]
    \centering
    \input{Prompts/Premiser_Prompt}
    \caption{\textbf{Prompt used for the Premise Generator.}}
    \label{fig:premiser_prompt}
\end{figure}

\begin{figure}[t]
    \centering
    \input{Prompts/Answerer_prompt}
    \caption{\textbf{Prompt used for the Answerer.}}
    \label{fig:answerer_prompt}
\end{figure}

\begin{figure}[t]
    \centering
    \input{Prompts/Clusterer_Prompt}
    \caption{\textbf{Prompt used for the Clusterer.}}
    \label{fig:Clusterer_prompt}
\end{figure}

\begin{figure}[t]
    \centering
    \input{Prompts/Ask4Conf_D}
    \caption{\textbf{Prompt used for Ask4Conf-D.}}
    \label{fig:ask4conf_prompt}
\end{figure}

\begin{figure}[t]
    \centering
    \input{Prompts/Quantified_Uncertainty_Based_Question_Clarification}
    \caption{\textbf{Prompt used for the Uncertainty-guided Clarification.}}
    \label{fig:Uncertainty-guided Clarification_prompt}
\end{figure}

\begin{figure}[t]
    \centering
    \input{Prompts/Localized_Uncertainty_Based_Question_Clarification}
    \caption{\textbf{Prompt used for the Uncertainty-guided Clarification with \proposed.}}
    \label{fig:Localized Uncertainty-guided Clarification_prompt}
\end{figure}

\begin{figure}[t]
    \centering
    \input{Prompts/Uncertainty_Context}
    \caption{\textbf{Uncertainty Context format used for the Uncertainty-guided Clarification.}}
    \label{fig:Uncertainty Context}
\end{figure}

\begin{figure}[t]
    \centering
    \input{Prompts/Localization_Context}
    \caption{\textbf{Fine-grained Uncertainty Context format used for the Uncertainty-guided Clarification with \proposed.}}
    \label{fig:Localized Context}
\end{figure}

\newpage
\begin{figure}
    \centering
    \input{Prompts/Medical_Localizer}
    \caption{\textbf{Prompt for the Localizer}: identifies ambiguous spans in the current utterance given the dialogue history.}
    \label{fig:medical_localizer}
\end{figure}

\begin{figure}
    \centering
    \input{Prompts/Medical_Premise_Generator}
    \caption{\textbf{Prompt for the Premise Generator}: produces mutually exclusive clinical interpretations for each ambiguous span.}
    \label{fig:medical_premise_generator}
\end{figure}
\begin{figure}
    \centering
    \input{Prompts/Medical_Answerer}
    \caption{\textbf{Prompt for the answerer}: rewrites the current utterance under a specific assumption about an ambiguous span.}
    \label{fig:medical_answerer}
\end{figure}
\begin{figure}
    \centering
    \input{Prompts/Medical_Clusterer}
    \caption{\textbf{Prompt for the Clusterer}: groups rewritten utterances by shared clinical or semantic meaning.}
    \label{fig:medical_cluster}
\end{figure}
\clearpage

\end{document}